\documentclass[twoside]{article}

\usepackage[accepted]{aistats2020}
\usepackage[round]{natbib}

\usepackage[utf8]{inputenc} 
\usepackage[T1]{fontenc}    
\usepackage{hyperref}       
\usepackage{url}            
\usepackage{booktabs}       
\usepackage{amsfonts}       
\usepackage{nicefrac}       
\usepackage{microtype}      

\usepackage{amsmath}		
\usepackage{amsthm}
\usepackage{amssymb}
\usepackage{bbold}
\usepackage[linesnumbered, lined, boxed]{algorithm2e}
\usepackage{mathtools}
\usepackage{enumerate}
\usepackage{pgfplots}
\pgfplotsset{compat=1.5}
\usepackage{subcaption}
\usepackage{nicefrac}
\usepackage{multirow}
\usepackage{dblfloatfix}
\usepackage{graphicx}
\usepackage{subcaption}
\usepackage{wrapfig}

\theoremstyle{definition}

\theoremstyle{theorem}
\newtheorem{theorem}{Theorem}[section]
\theoremstyle{lemma}
\newtheorem{lemma}{Lemma}[section]
\theoremstyle{remark}

\newcommand{\R}{\mathbb{R}}
\newcommand{\E}{\mathbb{E}}

\newcommand{\GG}{\mathcal{G}}
\newcommand{\HH}{\mathcal{H}}

\newcommand{\MM}{\mathcal{M}}
\newcommand{\NN}{\mathcal{N}}
\newcommand{\OO}{\mathcal{O}}
\newcommand{\RR}{\mathcal{R}}

\newcommand{\XX}{\mathcal{X}}
\newcommand{\YY}{\mathcal{Y}}
\newcommand{\ZZ}{\mathcal{Z}}

\newcommand{\TF}{\mathcal{TF}}

\newcommand{\Cov}{\mbox{Cov}}

\newcommand{\iid}{\stackrel{iid}{\sim}}

\newcommand{\pto}{\overset{p}{\to}}
\newcommand{\dto}{\overset{d}{\to}}

\usepackage[T1]{fontenc}    

\begin{document}

%

%
\runningauthor{Kirchler, Khorasani, Kloft, Lippert}

\twocolumn[

\aistatstitle{Two-sample Testing Using Deep Learning}

\aistatsauthor{Matthias Kirchler$^{1,2}$ \quad Shahryar Khorasani$^1$ \quad Marius Kloft$^{2,3}$ \quad Christoph Lippert$^{1,4}$}
\aistatsaddress{
$^1$Hasso Plattner Institute for Digital Engineering, University of Potsdam,  Germany
\\
$^2$Technical University of Kaiserslautern, Germany
\\
$^3$University of Southern California, Los Angeles, United States
\\
$^4$Hasso Plattner Institute for Digital Health at Mount Sinai, New York, United States
}
]

\begin{abstract}
We propose a two-sample testing procedure based on learned deep neural network representations.
To this end, we define two test statistics that perform an asymptotic location test on data samples mapped onto a hidden layer.  
The tests are consistent and asymptotically control the type-1 error rate.
Their test statistics can be evaluated in linear time (in the sample size).
Suitable data representations are obtained in a data-driven way, by solving a supervised or unsupervised transfer-learning task on an auxiliary (potentially distinct) data set.
If no auxiliary data is available, we split the data into two chunks: one for learning representations and one for computing the test statistic. 
In experiments on audio samples, natural images and three-dimensional neuroimaging data our tests yield significant decreases in type-2 error rate (up to 35 percentage points) compared to state-of-the-art two-sample tests such as kernel-methods and classifier two-sample tests.\footnote{We provide code at \url{https://github.com/mkirchler/deep-2-sample-test} }
\end{abstract}

\section{INTRODUCTION}
For almost a century, statistical hypothesis testing has been one of the main methodologies in statistical inference~\citep{neyman1933problem}.
A classic problem is to validate whether two sets of observations are drawn from the same distribution (null hypothesis) or not (alternative hypothesis).
This procedure is called \emph{two-sample test}.

Two-sample tests are a pillar of applied statistics and a standard method for analyzing empirical data in the sciences, e.g., medicine, biology, psychology, and social sciences. 
In machine learning, two-sample tests have been used to evaluate generative adversarial networks~\citep{binkowski2018demystifying},
to test for covariate shift in data~\citep{zhou2016hypothesis}, and to infer causal relationships~\citep{lopez2016revisiting}. 

There are two main types of two-sample tests: parametric and non-parametric ones.
Parametric two-sample tests, such as the Student's $t$-test, make strong assumptions on the distribution of the data (e.g. Gaussian).
This allows us to compute p-values in closed form.
However, parametric tests may fail when their assumptions on the data distribution are invalid.
Non-parametric tests, on the other hand, make no distributional assumptions and thus could potentially be applied in a wider range of application scenarios.
Computing non-parametric test statistics, however, can be costly as it may require applying re-sampling schemes or computing higher-order statistics.

A non-parametric test that gained a lot of attention in the machine-learning community is the kernel two-sample test and its 
test statistic: the maximum mean discrepancy (MMD). 
MMD computes the average distance of the two samples mapped into the reproducing kernel Hilbert space (RKHS) of a universal kernel (e.g., Gaussian kernel). 
MMD critically relies on the choice of the feature representation (i.e., the kernel function) and thus might fail for complex, structured data such as sequences or images, and other data where deep learning excels.

Another non-parametric two-sample test is the classifier two-sample test (C2ST).
C2ST splits the data into two chunks, training a classifier on one part and evaluating it on the remaining data.
If the classifier predicts significantly better than chance, the test rejects the null hypothesis. 
Since a part of the data set needs to be put aside for training, not the full data set is used for computing the test statistic,
which limits the power of the method. Furthermore, the performance of the method depends on the selection of the train-test split.

In this work, we propose a two-sample testing procedure that uses deep learning to obtain a suitable data representation.
It first maps the data onto a hidden-layer of a deep neural network that was trained (in an unsupervised or supervised fashion) on an independent, auxiliary data set,
and then it performs a location test.
Thus we are able to work on any kind of data that neural networks can work on, such as audio, images, videos, time-series, graphs, and natural language.
We propose two test statistics that can be evaluated in linear time (in the number of observations), based on MMD and Fisher discriminant analysis, respectively.
We derive asymptotic distributions of both test statistics.
Our theoretical analysis proves that the two-sample test procedure asymptotically controls the type-1 error rate, has asymptotically vanishing type-2 error rate and is robust both with respect to transfer learning and approximate training.

We empirically evaluate the proposed methodology in a variety of applications from the domains of computational musicology, computer vision, and neuroimaging.
In these experiments, the proposed deep two-sample tests consistently outperform the closest competing method (including deep kernel methods and C2STs) by up to 35 percentage points in terms of the type-2 error rate, while properly controlling the type-1 error rate.
\section{PROBLEM STATEMENT \& NOTATION}
\label{sec:setting}
We consider non-parametric two-sample statistical testing, that is, to answer the question whether two samples are drawn from the same (unknown) distribution or not.
We distinguish between the case that the two samples are drawn from the same distribution (the null hypothesis, denoted by $H_0$) and the case that the samples are drawn from different distributions (the alternative hypothesis $H_1$).

We differentiate between type-1 errors (i.e,rejecting the null hypothesis although it holds) and type-2 errors (i.e., not rejecting $H_0$ although it does not hold).
We strive for both the type-1 error rate to be upper bounded by some significance level $\alpha$, and the type-2 error rate to converge to 0 for unlimited data.
The latter property is called consistency and means that with sufficient data, the test can reliably distinguish between any pair of probability distributions.

Let $p, q, p'$ and $q'$ be probability distributions on $\R^d$ with common dominating Borel measure $\mu$.
We abuse notation somewhat and denote the densities with respect to $\mu$ also by $p, q, p'$ and $q'$.
We want to perform a two-sample test on data drawn from $p$ and $q$, i.e. we test the null hypothesis $H_0: p = q$ against the alternative $H_1: p \neq q$.
$p'$ and $q'$ are assumed to be in some sense similar to $p$ and $q$, respectively, and act as auxiliary task for tuning the test (the case of $p=p'$ and $q=q'$ is perfectly valid, in which case this is equivalent to a data splitting technique).

We have access to four (independent) sets $\XX_n, \YY_n, \XX_{n'}'$, and $\YY_{n'}'$ of observations drawn from $p, q, p'$, and $q'$, respectively. Here $\XX_n = \{X_1, \ldots, X_n\} \subset \R^d$ and $X_i \sim p$ for all $i$ (analogue definitions hold for $\YY_n, \XX_{n'}'$, and  $\YY_{n'}'$).
Empirical averages with respect to a function $f$ are denoted by $\overline{f(\XX_n)} := \frac{1}{n}\sum_{i=1}^n f(X_i)$. 

We investigate function classes of deep ReLU networks with a final $\tanh$ activation function:
\begin{align*}
    \TF_N & := \left\{ \tanh \circ  W_{D-1} \circ \sigma \circ \ldots \circ \sigma \circ W_1 : \R^d \to \R^H \right| 
    \\
    & W_1 \in \R^{H\times d}, W_j \in \R^{H\times H} \mbox{ for } j=2, \ldots, D-1,
    \\
    & \left. \prod_{j=1}^{D-1}||W_j||_{Fro} \leq \beta_N, D\leq D_N
     \right\}
\end{align*}
Here, the activation functions $\tanh$ and $\sigma(z) := \mbox{ReLU}(z) = \max(0, z)$ are applied elementwise, $||\cdot||_{Fro}$ is the Frobenius norm, $H = d + 1$ is the width and $D_N$ and $\beta_N$ are depth and weight restrictions onto the networks.
This can be understood as the mapping onto the last hidden layer of a neural network concatenated with a $\tanh$ activation.

\section{DEEP TWO-SAMPLE TESTING}
\label{sec:deeptesting}

In this section, we propose two-sample testing based on two novel test statistics, the \textbf{Deep Maximum Mean Discrepancy (DMMD)} and the \textbf{Deep Fisher Discriminant Analysis (DFDA)}.
The test asymptotically controls the type-1 error rate, and it is consistent (i.e., the type-2 error rate converges to 0).
Furthermore, we will show that consistency is preserved under both transfer learning on a related task, as well as only approximately solving the training step.

\subsection{Proposed Two-sample Test}
\label{sec:testingprocedure}

Our proposed test consists of the following two steps.
1. We train a neural network over an auxiliary \emph{training} data set. 2. We then evaluate the maximum mean discrepancy test statistic \citep{gretton2012kernel} (or a variant of it) using as kernel the mapping from the input domain onto the network's last hidden layer.

\subsubsection{Training Step}
Let the training data be $\XX_{n'}'$ and $\YY_{m'}'$. Denote $N = n' + m'$.
We run a (potentially inexact) training algorithm to find $\phi_N \in \TF_N$ with: 
\begin{align*}
    &  \left|\left| \frac{1}{N} \left( \sum_{i=1}^{n'} \phi_{N}(X_i') - \sum_{i=1}^{m'} \phi_{N}(Y_i') \right) \right|\right| + \eta
    \\
    & \ge \max_{\phi \in \TF_{N}} \left|\left| \frac{1}{N} \left( \sum_{i=1}^{n'} \phi(X_i') - \sum_{i=1}^{m'}\phi(Y_i') \right) \right|\right|.
\end{align*}
Here, $\eta \ge 0$ is a fixed leniency parameter (independent of $N$);
finding true global optima in neural networks is a hard problem, and an $\eta > 0$ allows us to settle with good-enough, local solutions. 
This procedure is also related to the early-stopping regularization technique, which is commonly used in training deep neural networks \citep{prechelt1998early}.

\subsubsection{Test Statistic}
We define the mean distance of the two test populations $\XX_n, \YY_m$ measured on the hidden layer of a network $\phi$ as 
\begin{align*}
    D_{n,m}(\phi) := \overline{\phi(\XX_n)} - \overline{\phi(\YY_m)}.
\end{align*}
Using $\phi_{N}$ from the training step, we define the Deep Maximum Mean Discrepancy (DMMD) test statistic as
\begin{align*}
    S_{n,m}(\phi_N, \XX_n, \YY_m) := \frac{nm}{n+m}\left|\left| D_{n,m}(\phi_N)\right|\right|^2.
\end{align*}
We can normalize this test statistic by the (inverse) empirical covariance matrix:
\begin{align*}
    &T_{n,m}(\phi_N, \XX_n, \YY_m) := \frac{nm}{n+m} D_{n,m}(\phi_N)^\top \hat{\Sigma}_{n,m}^{-1} D_{n,m}(\phi_N).
\end{align*}
This leads to a test statistic (which we call Deep Fisher Discriminant Analysis---DFDA) with an asymptotic distribution that is easier to evaluate.
Note that the empirical covariance matrix is defined as:
\begin{align*}
    &\hat{\Sigma}_{n,m} := \hat{\Sigma}_{n,m}(\phi_N) := 
    \\
    & \frac{1}{n+m-1}\sum_{i=1}^{m+n}(\phi_{N}(Z_i)- \overline{\phi_{N}(\ZZ)})(\phi_{N}(Z_i)- \overline{\phi_{N}(\ZZ)})^\top
    \\
    & + \rho_{n,m}I,
\end{align*}
where $\rho_{n,m} > 0$ is a factor guaranteeing numerical stability and invertibility of the covariance matrix, and $\ZZ = \{Z_1, \ldots, Z_{m+n}\} = \{ X_1, \ldots, X_n, Y_1, \ldots, Y_m\}$.

\subsubsection{Discussion}
\label{sec:testing-discussion}
Intuitively, we map the data onto the last hidden layer of the neural network and perform a multivariate location test on whether both map to the same location.
If the distance $D_{n,m}$ between the two means is too large, we reject the hypothesis that both samples are drawn from the same distribution.
Consistency of this procedure is guaranteed by the training step.

\paragraph{Interpretation as Empirical Risk Minimization} 
If we identify $X_i'$ with $(Z_i', 1)$ and $Y_{i}'$ with $(Z_{n'+i}', -1)$ in a regression setting, this is equivalent to an (inexact) empirical risk minimization with loss function $L(t, \hat{t}) = 1 - t\hat{t}$:
\begin{align*}
    \max_{\phi} \left|\left|\frac{1}{N}\sum_{i=1}^N t_i' \phi(Z_i')\right|\right|
    = \max_{\phi} \max_{||w||\leq 1} \frac{1}{N}\sum_{i=1}^N t_i' w^\top \phi(Z_i'),
\end{align*}
which is equivalent to
\begin{align}
    \label{eq:erm}
    \min_{\phi}\min_{||w||\leq 1} R_N'(w^\top\phi) := \frac{1}{N} \sum_{i=1}^N L(t_i',  w^\top \phi(Z_i')),
\end{align}
where we denote by $R_N'$ the empirical risk; the corresponding expected risk is $R'(f) = \E[1 - t'f(Z')]$.
Assuming that $\Pr(t'=1) = \Pr(t'=-1) = \frac{1}{2}$, we have for the Bayes risk $R'^* = \inf_{f:\R^d\to[-1,1]} R'(f) = 1 - \epsilon'$ with $\epsilon' > 0$ if and only if $p' \neq q'$.
As long as $p'$ and $q'$ are selected close enough to $p$ and $q$, respectively, the corresponding test will be able to distinguish between the two distributions.

Since we discard $w$ after optimization and use the norm of the hidden layer on the test set again, this implies some fine-tuning on the test data, without compromising the test statistic (see Theorem~\ref{thm-pvals} below).
This property is especially helpful in neural networks, since for practical transfer learning, only fine-tuning the last layer can be extremely efficient, even if the transfer and actual task are relatively different \citep{lu2015transfer}.

\paragraph{Relation to kernel-based tests}
The test statistic $S_{n,m}$ is a special case of the standard squared Maximum Mean Discrepancy \citep{gretton2012optimal} with the kernel $k(z_1, z_2) := \langle \phi(z_1), \phi(z_2) \rangle$ (analogously for $T_{n,m}$ and the Kernel FDA Test \citep{harchaoui2008testing}).
For a fixed feature map $\phi$ this kernel is not characteristic, and hence the resulting test not necessarily consistent for arbitrary distributions $p, q$.
However, by first choosing $\phi$ in a data-dependent way, we can still achieve consistency.


\subsection{Control of Type-1 Error}
Due to our choice of $\phi_N$, there need not be a unique, well-defined limiting distribution for the test statistics when $n,m\to\infty$.
Instead, we will show that for each \emph{fixed} $\phi$, the test statistic $S_{n,m}$ has a well-defined limiting distribution that can be well evaluated.
If in addition the covariance matrix is invertible, then the same holds for $T_{n,m}$.

In particular, the following theorem will show that $D_{n,m}(\phi)$ converges towards a multivariate normal distribution for $n,m\to\infty$.
$S_{n,m}$ then is asymptotically distributed like a weighted sum of $\chi^2$ variables, and $T_{n,m}$ like a $\chi^2_H$ (again, if well-defined).

\begin{theorem}
\label{thm-pvals}
Let $p = q$, $\phi\in\TF$ and $\Sigma := \Cov(\phi(X_1))$ and assume that $\frac{n}{n+m}\to r\in (0,1)$ as $n, m \to \infty$.
\begin{itemize}
    \item[(i)] As $n, m \to \infty$, it holds that
        \begin{align*}
            \sqrt{\frac{mn}{m+n}}D_{n,m}(\phi) \dto \NN(0, \Sigma).
        \end{align*}
    \item[(ii)]  As $n, m \to \infty$,
        \begin{align*}
            S_{n,m}(\phi, \XX_n, \YY_m) \dto \sum_{i=1}^H \lambda_i \xi_i^2,
        \end{align*}
        where $\xi_i \iid \NN(0,1)$ and $\lambda_i$ are the eigenvalues of $\Sigma$.
    \item[(iii)] If additionally $\Sigma$ is invertible, and $\rho_{n,m} \downarrow 0$ then as $n,m\to\infty$
        \begin{align*}
            T_{n,m}(\phi, \XX_n, \YY_m) \dto \chi_H^2.
        \end{align*}
\end{itemize}
\end{theorem}
\begin{proof}[Sketch of proof (full proof in Appendix~\ref{app-proofs-pvals})]
\emph{(i)} As under $H_0$ $\phi(X_i)$ and $\phi(Y_j)$ are identically distributed, $D_{n,m}(\phi)$ is centered and one can show the result using a Central Limit Theorem.

\emph{(ii)} and \emph{(iii)} then follow from the continuous mapping theorem and properties of the multivariate normal distribution.
\end{proof}

Under some additional assumptions we can also use a Berry-Esseen type of result to quantify the quality of the normal approximation of $D_{n,m}(\phi_N)$ conditioned on the training.
In particular, if we assume that $n = m$ and $\Sigma = \Cov_{p,q}(\phi_N(X_1)) | \XX_n', \YY_n'$ invertible, then \citet{bentkus2005lyapunov} shows that the normal approximation on convex sets is $\OO\left(\frac{H^{1/4}}{\sqrt{n}}\right)$.
Computing p-values for both $S_{n,n}$ and $T_{n,n}$ only requires computation over convex sets, so the result is directly applicable.

\subsubsection{Computational Aspects}
\label{sec-computationpvals}
\paragraph{Testing with $S_{n,m}$} As shown in Theorem~\ref{thm-pvals}, the null distribution of $S_{n,m}$ can be approximated as the weighted sum of independent $\chi^2$-variables.
There are several approaches to computing the cumulative distribution function of this distribution, see \citet{bausch2013efficient} for an overview and \citet{zhou2018null} for an implementation.
However, computing p-values with this method can be rather costly.

Alternatively, note that the test statistic $S_{n,m}$ is linear in the number of observations and dimensions.
Hence, estimating the null distribution via Monte-Carlo permutation sampling \citep{ernst2004permutation} is feasible.
Note also that it suffices to evaluate the feature map $\phi$ on each data point only once and then permute the class labels, saving more time.

In practice we found that the resampling-based test performed considerably faster.
Hence, in the remainder of this work, we will evaluate the null hypothesis of the DMMD via the resampling method.

\paragraph{Testing with $T_{n,m}$}
Since in many practical situations one wants to use standard neural network architectures (such as ResNets), the number of neurons in the last hidden layer $H$ may be rather large, compared to $n, m$.
Therefore, using the full, high-dimensional hidden layer representation might lead to suboptimal normal approximations.
Instead, we propose to use a principal component analysis on the feature representation $(\phi(Z_i))_{i=1}^{n+m}$ to reduce the dimensionality to $\hat{H} \ll m+n$.
In fact, this does not break the asymptotic theory derived in Theorem~\ref{thm-pvals}, even though the PCA is both trained and evaluated on the test data; details can be found in Appendix~\ref{app-pca}.
Unfortunately, the $\OO\left(\frac{H^{1/4}}{\sqrt{n}}\right)$ rate of convergence is not valid anymore, due to the observations not being independent.
We still need to grow $\hat{H}$ towards $H$ with $n, m$ in order for the consistency results in the next section to hold, however.
Empirically we found $\hat{H} = \min\left(\sqrt{\frac{n+m}{2}}, H\right)$ to perform well.

The cumulative distribution function of the $\chi_H^2$ distribution can be evaluated very efficiently.
Although for the DFDA it is also possible to estimate the null hypothesis via a Monte Carlo permutation scheme, doing so is more costly than for the DMMD, since it involves either a matrix inversion once or solving a linear system for each permutation draw.
Hence, in this work we focus on using the asymptotic distribution.

\subsection{Consistency}
In this section we show that if $(a)$, the restrictions $\beta_N, D_N$ on weights and depth of networks in $\TF_N$ are carefully chosen, $(b)$, the transfer task is not too far from the original task, and $(c)$, the leniency parameter $\eta$ in the training step is small enough, then our proposed test is consistent, meaning the type-2 error rate converges to 0.

\begin{theorem}
\label{thm-consistency}
    Let $p \neq q$, $n=n', m=m'$ with $\frac{n}{m}\to 1$, $N = n + m$, $R'^* = 1 - \epsilon'$ the Bayes error for the transfer task with $\epsilon' > 0$, and assume that the following holds:
    \begin{itemize}
        \item[(i)] $\frac{\beta_N^2 D_N}{N} \to 0$, $\beta_N \to \infty$ and $D_N \to \infty$ for $N\to\infty$ for the parameters of the function classes $\TF_N$,
        \item[(ii)] $||p - p'||_{L_1(\mu)} + ||q - q'||_{L_1(\mu)} \leq 2 \delta$,
        \item[(iii)] $0 \leq \delta + \eta < \epsilon'$, where $\eta\geq 0$ is the leniency parameter in training the network, and
        \item[(iv)] $p'$ and $q'$ have bounded support on $\R^d$.\footnote{A similar Theorem holds also for the case of unbounded support, see Appendix~\ref{sec:unbounded}}
    \end{itemize}
    Then, as $N\to\infty$ both test test statistics $S_{n,m}(\phi_N, \XX_n, \YY_m)$ and $T_{n,m}(\phi_N, \XX_n, \YY_m)$ diverge in probability towards infinity, i.e. for any $r>0$
    \begin{align*}
        & \Pr\left(S(\phi_N, \XX_n, \YY_m) > r\right) \to 1 \mbox{ and }
        \\
        & \Pr\left(T(\phi_N, \XX_n, \YY_m) > r\right) \to 1.
    \end{align*}
\end{theorem}
\begin{proof}[Sketch of proof (full proof in Appendix~\ref{app-proofs-consistency})]
    The test statistics $S_{n,m}$ is lower-bounded by a rescaled version of $\sqrt{N}(1 - R_{n,m}(\psi_N))$, where $\psi_N = w_N^\top \phi_N$ with $w_N$ selected as in \eqref{eq:erm}.
    Then, if $1 - R_{n,m}(\psi_N) \geq c > 0$, the test statistic diverges.

    The finite-sample error $R_{n,m}(\psi_N)$ approaches its population version $R(\psi_N)$ for large $n, m$, and the difference between $R(\psi_N)$ and $R'(\psi_N)$ can be controlled over $\delta$.
    The rest of the proof is akin to standard consistency proofs in regression and classification.
    Namely, we can split $R_N'(\psi_N) - R'^*$ into approximation and estimation error and control these via a Universal Approximation Theorem \citep{hanin2017universal}, and Rademacher complexity bounds on the neural network function class \citep{golowich2017size}, respectively.
\end{proof}

The main caveat of Theorem~\ref{thm-consistency} is that it gives no explicit directions to choose the transfer task $p'$ and $q'$.
Whether the respective $\mu$-densities are $L_1$-close to the testing densities in general cannot be answered, and similarly the Bayes error rate $1- \epsilon'$ is not known beforehand.
If abundant data for the testing task is at hand, then splitting the data is the safe way to go; if data is scarce, Theorem~\ref{thm-consistency} gives justification that a \emph{reasonably close} transfer task will have good power as well.

The bounded support requirement \emph{(iv)} on $p'$ and $q'$ can be circumvented as well -- by choosing the support large enough one can always just truncate $(X_i')$ and $(Y_i')$ and will still satisfy requirements \emph{(ii)} and \emph{(iii)}, especially also in the case of $p' = p$ and $q' = q$ with unbounded support.
This procedure, however, requires knowledge of where to truncate the transfer distributions.
Instead one can also grow the support of $p'$ and $q'$ with $N$; for more details, see Appendix~\ref{sec:unbounded}.

\section{RELATED WORK}

In this section, we give an overview over the state-of-the-art in non-parametric two-sample testing for high-dimensional data.

\paragraph{Kernel Methods}
The methods most related to our method are the kernelized maximum mean discrepancy (MMD) \citep{gretton2012kernel} and the kernel Fisher discriminant analysis (KFDA) \citep{harchaoui2008testing}.
Both methods effectively metricize the space of probability distributions by mapping distribution features onto mean embeddings in universal reproducing kernel Hilbert spaces (RKHS, \citep{steinwart2008support}).
Test statistics derived from these mean embeddings can be efficiently evaluated using the kernel trick (in quadratic time in the number of observations, although there are lower-powered linear-time variations).
Mean Embeddings (ME) and Smoothed Characteristic Functions (SCF) \citep{chwialkowski2015fast, jitkrittum2016interpretable} 
are kernel-based linear-time test statistics that are (almost surely) proper metrics on the space of probability distributions.
All four methods rely on characteristic kernels to yield consistent tests and are closely related.

\paragraph{Deep Kernel Methods}
In the context of training and evaluating Generative Adversarial Networks (GANs), several authors have investigated the use of the MMD with kernels parametrized by deep neural networks.
In \citet{binkowski2018demystifying, li2017mmd, arbel2018gradient}, the authors feed features extracted from deep neural networks into characteristic kernels.
\citet{jitkrittum2018informative} use deep kernels in the context of relative goodness-of-fit testing without directly considering consistency aspects of this approach.
Extensions from the GAN literature to two-sample testing is not straightforward since statistical consistency guarantees strongly depend on careful selection of the respective function classes.
To the best of our knowledge, all previous works made simplifying assumptions on injectivity or even invertibility of the involved networks.

In this work we show that a linear kernel on top of transfer-learned neural network feature maps (as has also been done by \citet{xu2018empirical} for GAN evaluation) is not only sufficient for consistency of the test, but also performs considerably better empirically in all settings we analyzed.
In addition to that, our test statistics can be directly evaluated in linear instead of quadratic time (in the sample size) and the corresponding asymptotic null distributions can be exactly computed (in contrast to the MMD \& KFDA).

\paragraph{Classifier Two-Sample Tests (C2ST)} First proposed by \citet{friedman2003multivariate} and then further analyzed by \citet{kim2016classification} and \citet{lopez2016revisiting}, the idea of the C2ST is to utilize a generic classifier, such as a neural network or a $k$-nearest neighbor approach for the two-sample testing problem.
In particular, they split the available data into training and test set, train a classifier on the training set and evaluate whether the performance on the test set exceeds random variation.
The main drawback of this approach is that the data has to be split in two chunks, creating a trade-off: if the training set is too small, the classifier is unlikely to find a statistically relevant signal in the data; if the training set is large and thus the test set small, the C2ST test loses power.

Our method circumvents the need to split the data in training and test set -- Theorem~\ref{thm-consistency} shows that training on a reasonably close transfer data set is sufficient.
Even more, as shown in Section~\ref{sec:testing-discussion}, our method can be interpreted as empirical risk minimization with additional fine-tuning of the last layer on the testing data, guaranteed to be as least as good as an equivalent method with fixed last layer.

\section{EXPERIMENTS}

\begin{figure*}[ht]
\centering
    \begin{subfigure}[b]{0.39\textwidth}
        \includegraphics[width=\textwidth]{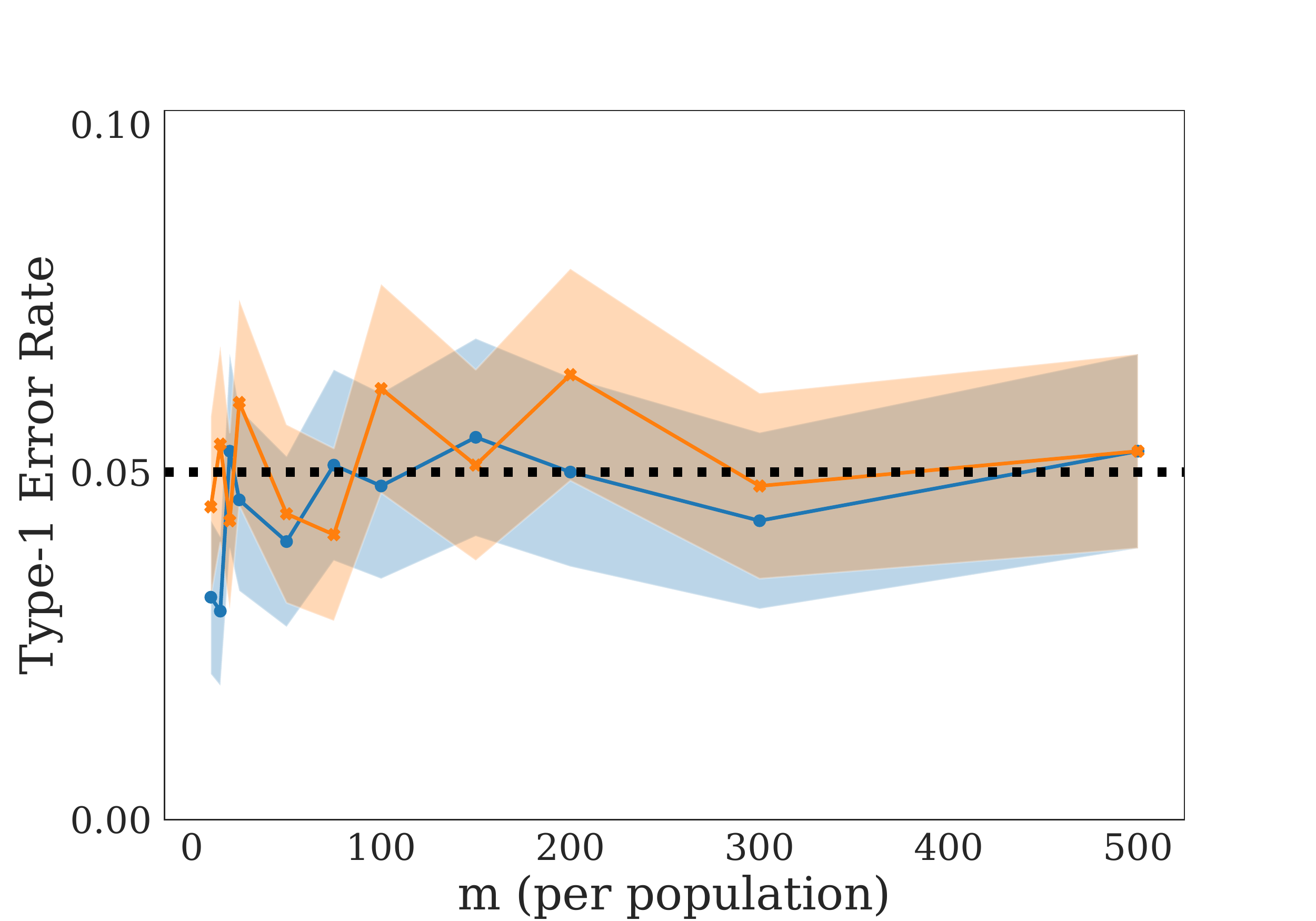}
        \caption{Type-1 error rate on AM audio data.}
        \label{fig:t1er}
    \end{subfigure}
    \begin{subfigure}[b]{0.39\textwidth}
        \includegraphics[width=\textwidth]{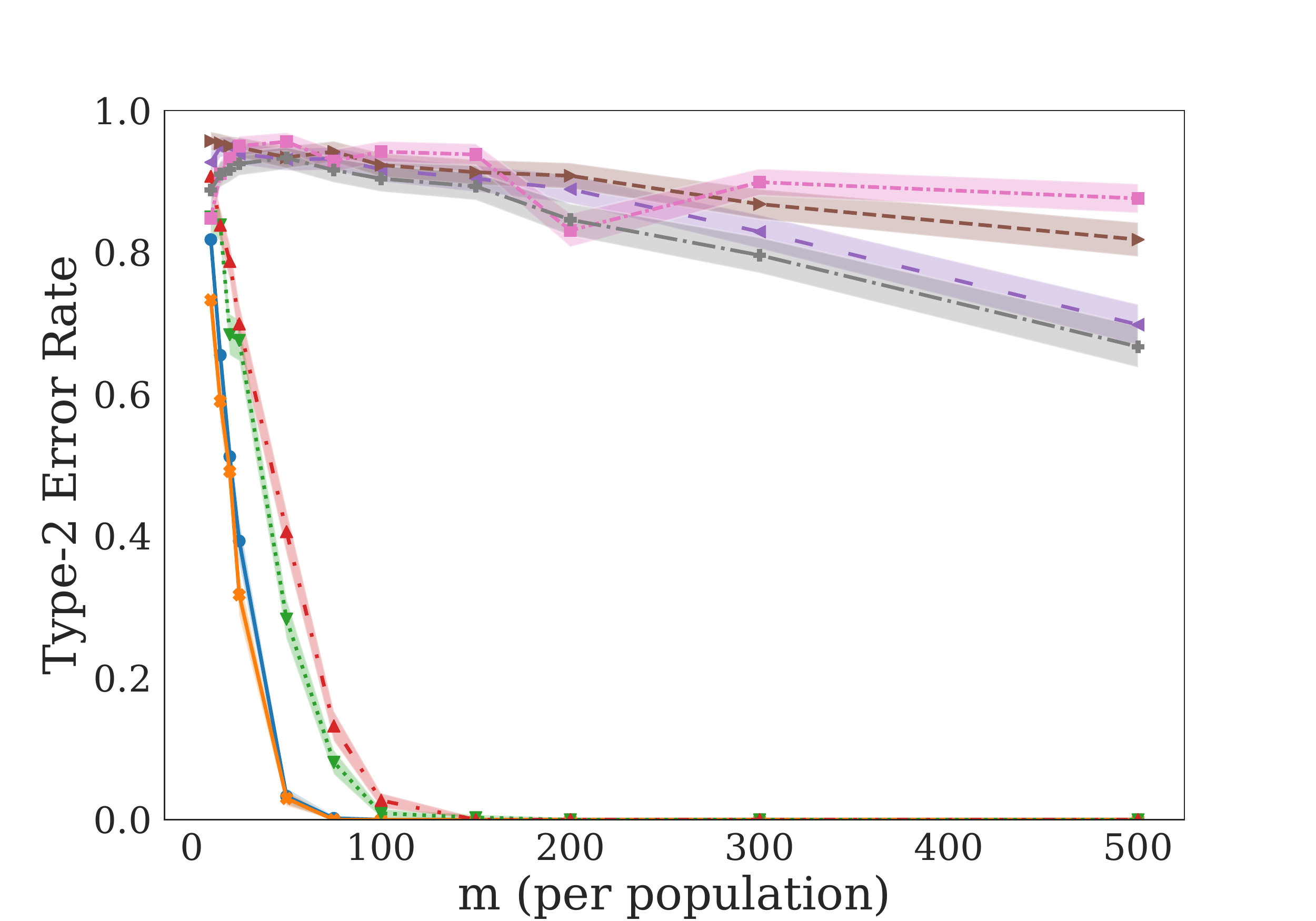}
        \caption{Type-2 error rate on AM audio data.}
        \label{fig:audio}
    \end{subfigure}
    \begin{subfigure}[b]{0.19\textwidth}
        \includegraphics[width=\textwidth]{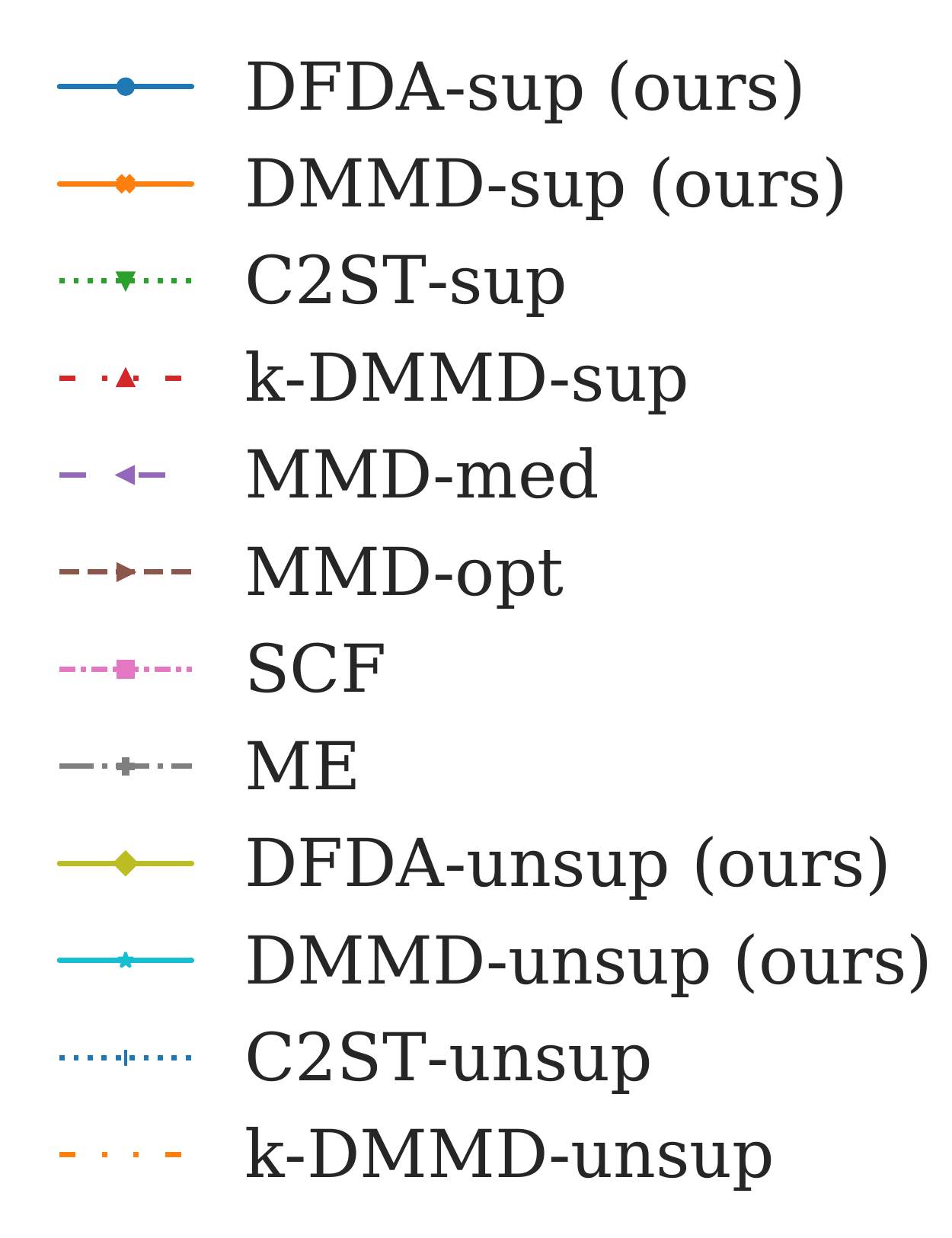}
        \vspace{1mm}
    \end{subfigure}
    \\
    \begin{subfigure}[b]{0.32\textwidth}
        \includegraphics[width=\textwidth]{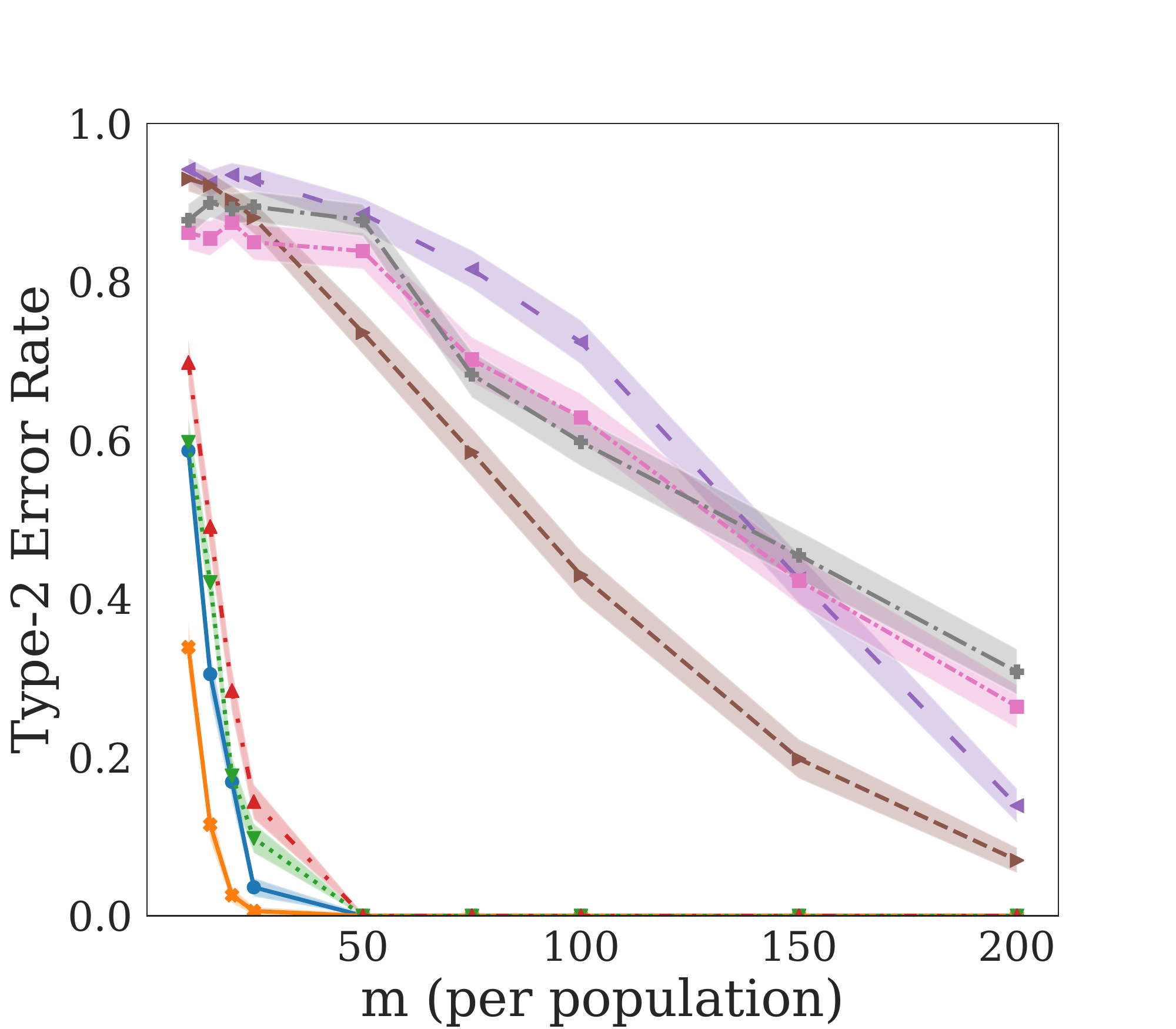}
        \caption{Type-2 error rate on aircraft data.}
        \label{fig:planes}
    \end{subfigure}
    \begin{subfigure}[b]{0.32\textwidth}
        \includegraphics[width=\textwidth]{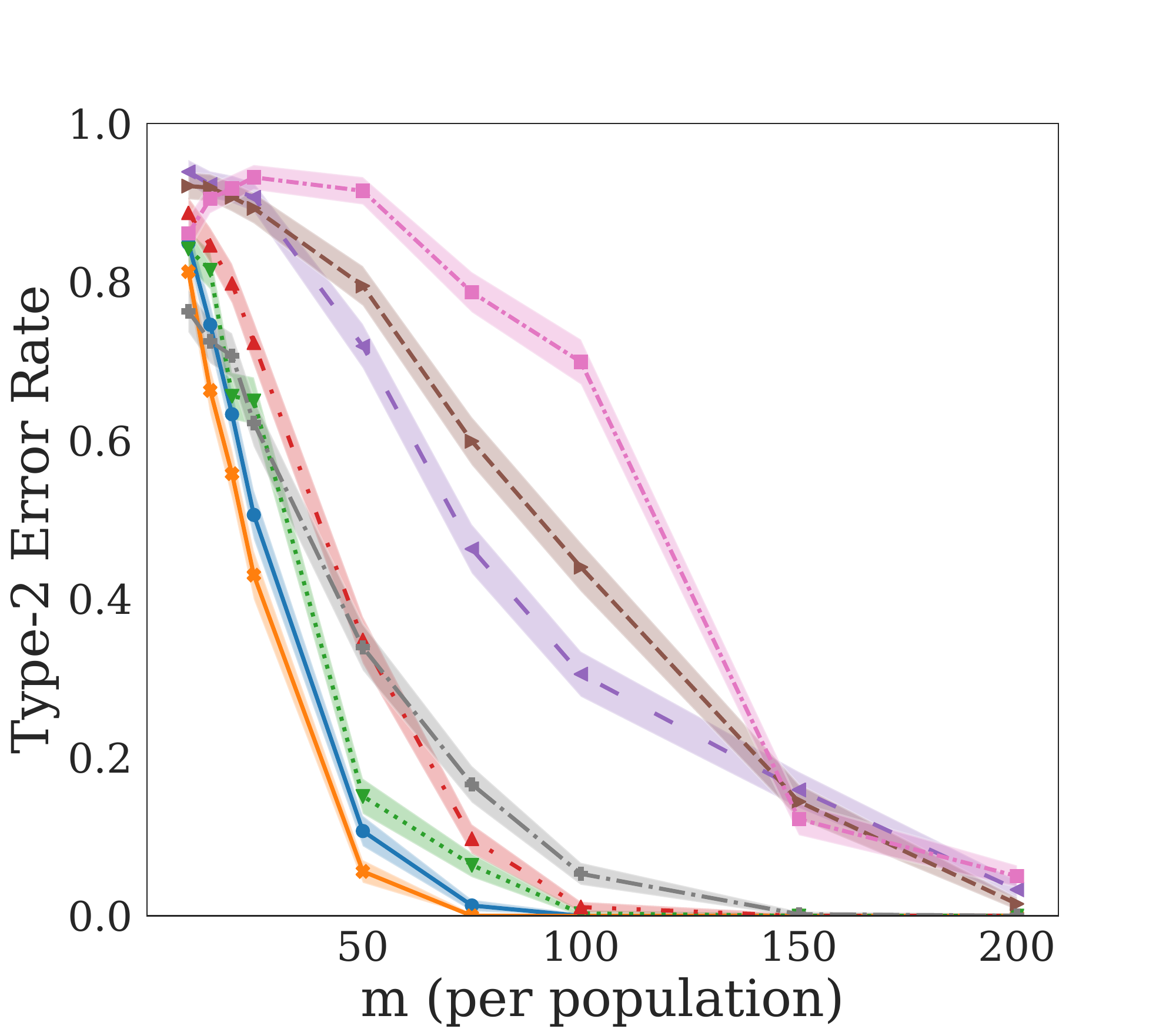}
        \caption{Type-2 error rate on KDEF data.}
        \label{fig:faces}
    \end{subfigure}
    \begin{subfigure}[b]{0.32\textwidth}
        \includegraphics[width=\textwidth]{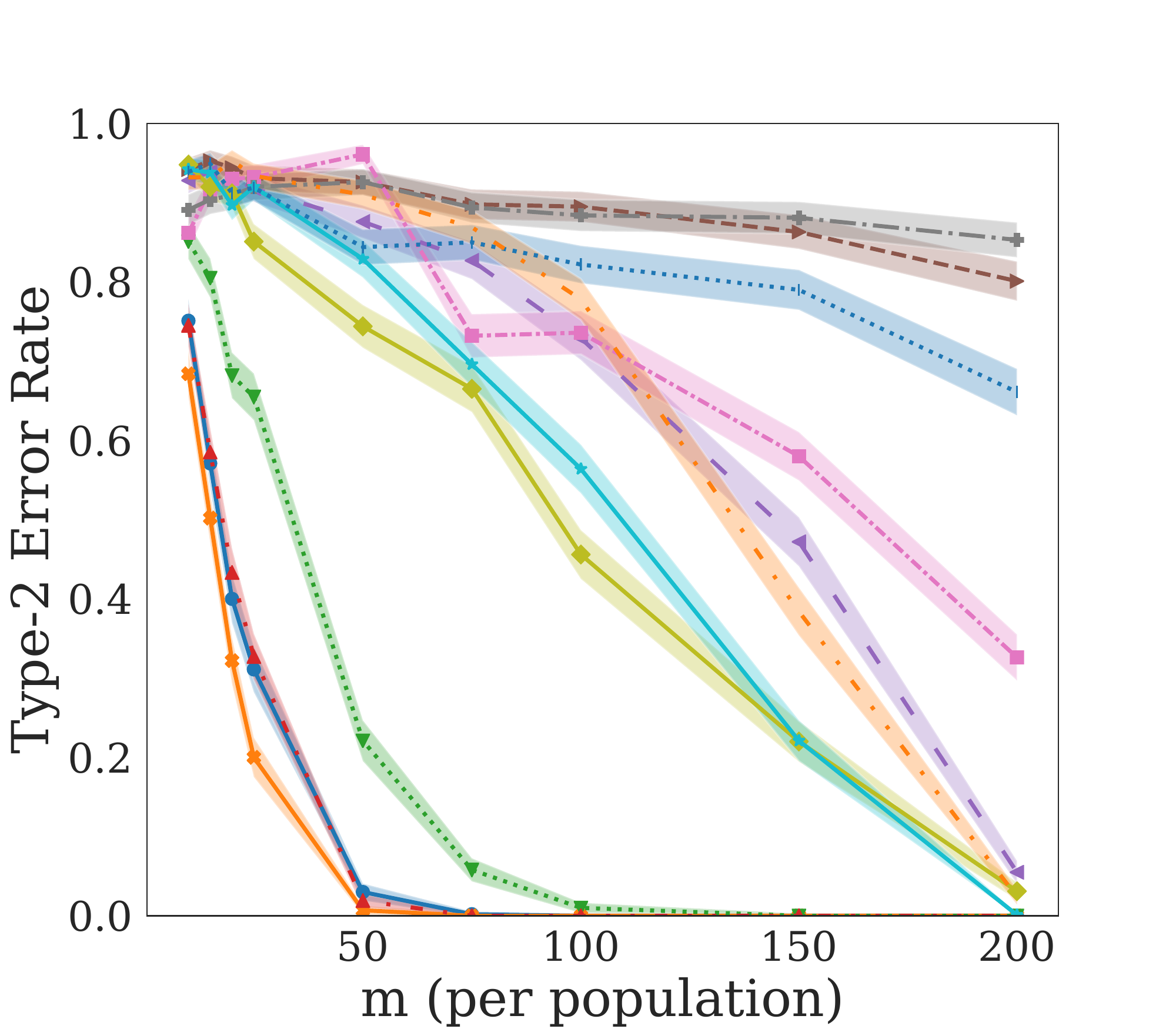}
        \caption{Type-2 error rate on dogs data.}
        \label{fig:dogs}
    \end{subfigure}
\caption{Results on AM audio (top row) and natural image (bottom row) data sets. Suffixes ``-sup`` indicate supervised pretraining, ``-unsup`` indicates unsupervised pretraining.}
\end{figure*}

In this section, we compare our proposed deep learning two-sample tests with other state-of-the-art approaches.

\subsection{Experimental setup}
\label{sec:expsetup}
For the \textbf{DFDA} and \textbf{DMMD} tests we train a deep neural network on a related task; details will be deferred to the corresponding sections.
We report both the performance of the deep MMD $S_{n,m}$ where we estimate the null hypothesis via a Monte Carlo permutation sample \citep{ernst2004permutation} (we fix $M=1000$ resampling permutations except otherwise noted), and the deep FDA statistic $T_{n,m}$, for which we use the asymptotic $\chi_H^2$ distribution.
As explained in Section~\ref{sec-computationpvals}, for the DFDA we project the last hidden layer onto $\hat{H} < H$ dimensions using a PCA.
We found the heuristic $\hat{H} := \sqrt{\frac{m+n}{2}}$ to perform well across a number of tasks (disjoint from the ones presented in this section).
For the DMMD we do not need any dimensionality reduction.
We calibrated parameters of both tests on data disjoint from the ones that we report results on in the subsequent sections.

For the \textbf{C2ST}, we train a standard logistic regression on top of the pretrained features extracted from the same neural network as for our methods.

For the \textbf{kernel MMD} we report two kernel bandwidth selection strategies for the Gaussian kernel.
The first variant is the ``median distance`` heuristic \citep{gretton2012kernel} which selects the median of the euclidean distances of all data points (MMD-med).
The second variant, reported by \citet{gretton2012optimal}, splits the data in two disjoint sets and selects the bandwidth that maximizes power on the first set and evaluates the MMD on the second set (MMD-opt).
We use the implementation provided by \citet{jitkrittum2016interpretable}, which estimates the null hypothesis via a Monte Carlo permutation scheme (we again use $M=1000$ permutations).

For the \textbf{Smoothed Characteristic Functions} (SCF) and \textbf{Mean Embeddings} (ME), we select the number of test locations based on the task and sample size.
The locations are selected either randomly (as presented by \citet{chwialkowski2015fast}) or optimized on half of the data via the procedure described by \citet{jitkrittum2016interpretable}.
The kernel was either selected using the median heuristic, or via a grid search as by \citet{chwialkowski2015fast, jitkrittum2016interpretable}.
In each case we report the kernel and location selection method that performed best on the given task, with details given in the corresponding paragraphs.
Note that for very small sample sizes, both SCF and ME oftentimes do not control the type-1 error rate properly, since they were designed for larger sample sizes.
This results in highly variable type-2 error rate for small $m$ in the experiments.
Again, we use the implementation provided by \citet{jitkrittum2016interpretable}.

In addition to these published methods, we also compare our method against a \textbf{deep kernel MMD test} (k-DMMD), i.e. the MMD test where the output of a pretrained neural network gets fed into a Gaussian kernel (instead of a linear kernel as in our case).
\citet{jitkrittum2018informative} used this method for relative goodness-of-fit testing instead of two-sample testing. 
For image data, we select the bandwidth parameter for the Gaussian kernel via the median heuristic, and for audio data via the power maximization technique (in each case the other variant performs considerably worse); the pretrained networks are the same as for our tests and the C2ST.

All experiments were run over 1000 runs.
Type-1 error rates are estimated by drawing both samples (without replacement) from the same class and computing the rate of rejections.
Similarly, type-2 error rates are estimated as the rate of not rejecting the null hypothesis when sampling from two distinct classes.
All figures of type-1 and type-2 error rates show the 95\% confidence interval based on a Wilson Score interval (and a ``rule-of-three`` approximation in the case of 0-values \citep{eypasch1995probability}).
In all settings we fixed the significance level at $\alpha = 0.05$.
In addition to that we show in Appendix~\ref{sec:sensitivity} empirically that also for smaller significance levels high power can be preserved.
Preprocessing for image data is explained in Appendix~\ref{sec:preprocessing}.

\subsection{Control of Type-1 Error Rate}
Since the presented test procedures are not exact tests it is important to verify that the type-1 error rate is controlled at the proper level.
Figure~\ref{fig:t1er} shows that the empirical type-1 error rate is well controlled for the amplitude modulated audio data introduced in the next section.
For the other data sets, results are provided in Appendix~\ref{sec:t1er}.

\subsection{Power Analysis}
\label{sec:power}
\paragraph{Amplitude Modulated Audio Data}

Here we analyze the proposed test on the amplitude modulated audio example from \citep{gretton2012optimal}.
The task in this setting is to distinguish snippets from two different songs after they have been amplitude modulated (AM) and mixed with noise.
We use the same preprocessing and amplitude modulation as \citet{gretton2012optimal}.
We use the freely available music from \citet{gramatiktaor}; distribution $p$ is sampled from track four, distribution $q$ from track five and the remaining tracks on the album were used for training the network in a multi-class classification setting.
As our neural network architecture we use a simple convolutional network, a variant from \citet{dai2017very}, called M5 therein; see Appendix~\ref{sec:audio-appendix} for details.

Figure~\ref{fig:audio} reports the results with varying number of observations under constant noise level $\sigma^2 = 1$. 
Our method shows high power, even at low sample sizes, whereas kernel methods need large amounts of data to deal with the task.
Note that these results are consistent with the original results in \citet{gretton2012optimal}, where the authors fixed the sample size at $m=10,000$ and consequently only used the (significantly less powerful) linear-time MMD test.

\paragraph{Aircraft} We investigate the Fine-Grained Visual Classification of Aircraft data set \citep{maji13fine-grained}.
We select two visually similar aircraft families, namely Boeing 737 and Boeing 747 as populations $p$ and $q$, respectively.
The neural network embeddings are extracted from a ResNet-152 \citep{he2016deep} trained on ILSVRC \citep{ILSVRC15}.
Figure~\ref{fig:planes} shows that all neural network architectures perform considerably better than the kernel methods.
Furthermore, our proposed tests can also outperform both the C2ST and the deep kernel MMD.

\paragraph{Facial Expressions} The Karolinska Directed Emotional Faces (KDEF) data set \citep{lundqvist1998karolinska} has been previously used by \citet{jitkrittum2016interpretable, lopez2016revisiting}.
The task is to distinguish between faces showing positive (happy, neutral, surprised) and negative (afraid, angry, disgusted) emotions.
The feature embeddings are again obtained from a ResNet-152 trained on ILSVRC.
Results can be found in Figure~\ref{fig:faces}.
Even though the images in ImageNet and KDEF are very different, the neural network tests again outperform the kernel methods.
Also note that the apparent advantage of the mean embedding test for low sample sizes is due to an unreasonably high type-1 error rate ($>0.11$ and $>0.085$ at $m=10, 15$, respectively).

\begin{table}[t]
    \caption{Results on neuroimaging data, comparing subjects who are cognitive normal (CN), have mild cognitive impairment (MCI) or have Alzheimer's disease (AD).
    \emph{APOE} has neutral variant $\varepsilon 3$ and risk-factor variant $\varepsilon 4$.
    Numbers in parentheses denote sample size.
    }
  \label{tab:mri_results}
  \centering
  \begin{tabular}{lll}
    \toprule
    X (\# obs) & Y (\# obs) & p-value \\
    \midrule
CN (490) & AD (314) & $9.49 \cdot 10^{-5}$\\
CN (490) & MCI (287) & $2.44 \cdot 10^{-4}$ \\ 
MCI (287) & AD (314) & $1.45 \cdot 10^{-3}$\\
    \midrule
\emph{APOE} $\varepsilon 3$  (811) & \emph{APOE} $\varepsilon 4$ (152) & $1.40 \cdot 10^{-2}$\\
    \bottomrule
  \end{tabular}
\end{table}

\paragraph{Stanford Dogs}
Lastly, we evaluate our tests on the Stanford Dogs data set \citep{KhoslaYaoJayadevaprakashFeiFei_FGVC2011}, consisting of 120 classes of different dog breeds.
As test classes we select the dog breeds `Irish wolfhound` and `Scottish deerhound`, two breeds that are visually extremely similar.
Since the data set is a subset of the ILSVRC data, we cannot train the networks on the whole ImageNet data again.
Instead, we train a small 6-layer convolutional neural network on the remaining 118 classes in a multi-class classification setting and use the embedding from the last hidden layer.
To show that our tests can also work with unsupervised transfer-learning, we also train a convolutional autoencoder on this data; the encoder part is identical to the supervised CNN, see Appendix~\ref{sec:dogs-appendix} for details.
Note that for this setting, the theoretical consistency guarantees from Theorem~\ref{thm-consistency} do not hold, although the type-1 error rate is still asymptotically controlled.
Figure~\ref{fig:dogs} reports the results, with *-sup denoting the supervised, and *-unsup the unsupervised transfer-learning task.
As expected, tests based on the supervised embedding approach outperform other tests by a large margin.
However, the unsupervised DMMD and DFDA still outperform kernel-based tests.
Interestingly, both the C2ST and the k-DMMD method seem to suffer more severely from the mediocre feature embedding than our tests.
One potential explanation for this phenomenon is the ability of DMMD and DFDA to fine-tune on the test data without the need to perform a data split.

\paragraph{Three-dimensional Neuroimaging Data}

\begin{figure}[t]
\centering
\includegraphics[width=0.95\linewidth]{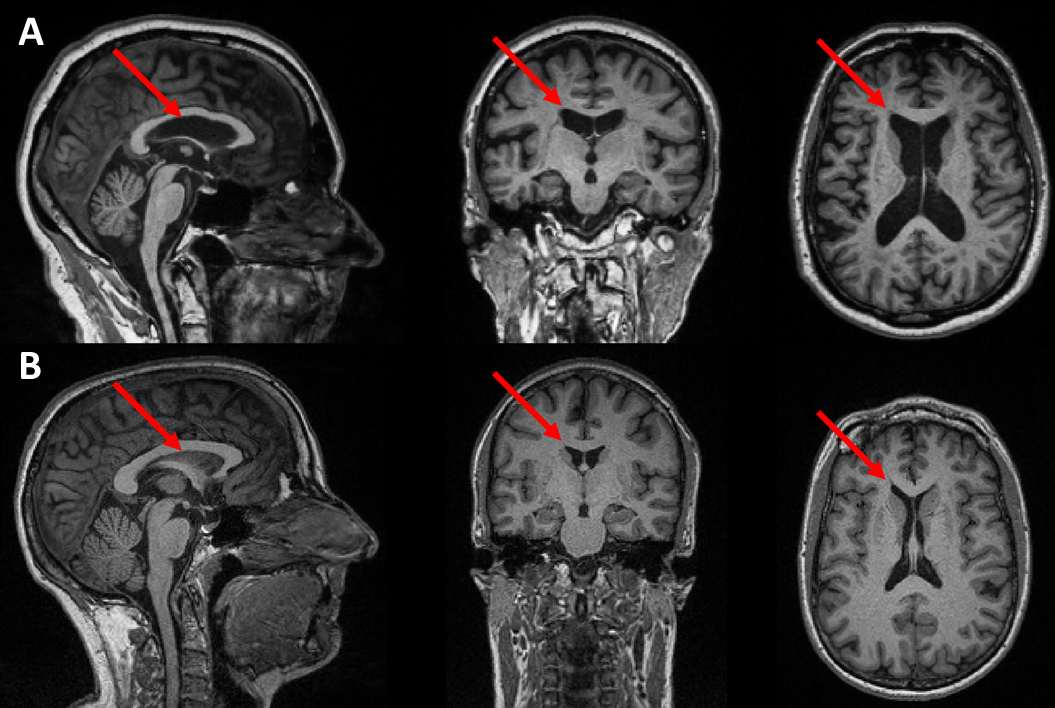}
    \caption{Slices of 3D-MRI scans of an Alzheimer's disease patient (A) and a cognitively normal individual (B). Note the enlargement of the lateral ventricles (indicated by red arrows) in the Alzheimer's disease patient.}
    \label{fig:mri}
\end{figure}
In this section, we apply the DFDA test procedure to 3D Magnetic Resonance Imaging (MRI) scans and genetic information from the Alzheimer’s Disease Neuroimaging Initiative (ADNI) \citep{mueller2005alzheimer}. 
To this end, we transfer a 3D convolutional autoencoder that has been trained on MRI scans from the Brain Genomics Superstruct Project \citep{holmes2015brain} to perform statistical testing on the ADNI data.
Details on preprocessing and network architecture are provided in Appendix~\ref{sec:mri-appendix}.

The ADNI dataset consists of individuals diagnosed with Alzheimer's Disease (AD), with Mild Cognitive Impairment (MCI), or as cognitively normal (CN); Figure~\ref{fig:mri} shows exemplaric images of an AD and a CN subject.
Table~\ref{tab:mri_results} shows that our test can detect statistically significant differences between MRI scans of individuals with a different diagnosis.
Additionally, we evaluate whether our test can detect differences between individuals who have a known genetic risk factor for neurodegenerative diseases and individuals without that risk factor.
In particular, we compare the two variants $\varepsilon 3$ (the ``normal'' variant) and $\varepsilon 4$ (the risk-factor variant) in the Apolipoprotein E (\emph{APOE}) gene, which is related to AD and other diseases \citep{corder1993apoe}.
By grouping subjects according to which variant they exhibit we test for statistical dependence between a (binary) genetic mutation and (continuous) variation in 3D MRI scans.
Table~\ref{tab:mri_results} shows that individuals with $\varepsilon 4$ and $\varepsilon 3$ \emph{APOE} variants are significantly different, suggesting a statistical dependence between genetic variation and structural brain features.

\subsection*{Acknowledgements}
The authors thank Stefan Konigorski and Jesper Lund for helpful discussions and comments.
Marius Kloft acknowledges support by the German Research Foundation (DFG) award KL 2698/2-1 and by the Federal Ministry of Science and Education (BMBF) awards 031L0023A, 01IS18051A, and 031B0770E.
Part of the work was done while Marius Kloft was a sabbatical visitor of the DASH Center at the University of Southern California.
This work has been funded by the Federal Ministry of Education and Research (BMBF, Germany) in the project KI-LAB-ITSE (project number 01|S19066).

Data used in the preparation of this article were obtained from the Alzheimer’s Disease Neuroimaging Initiative (ADNI) database \url{adni.loni.usc.edu}. As such, the investigators within the ADNI contributed to the design and implementation of ADNI and/or provided data but did not participate in analysis or writing of this report. A complete listing of ADNI investigators can be found at: \url{http://adni.loni.usc.edu/wp-content/uploads/how_to_apply/ADNI_ Acknowledgement_List.pdf}. Data collection and sharing of ADNI was funded by the Alzheimer’s Disease Neuroimaging Initiative (ADNI) (National Institutes of Health Grant U01 AG024904) and DOD ADNI (Department of Defense award number W81XWH-12-2-0012). ADNI is funded by the National Institute on Aging, the National Institute of Biomedical Imaging and Bioengineering, and through generous contributions from the following: Alzheimer’s Association; Alzheimer’s Drug Discovery Foundation; BioClinica Inc; Biogen Idec Inc; Bristol-Myers Squibb Company; Eisai Inc; Elan Pharmaceuticals Inc; Eli Lilly and Company; F. Hoffmann-La Roche Ltd and its affiliated company Genentech Inc; GE Healthcare; Innogenetics N.V.; IXICO Ltd; Janssen Alzheimer Immunotherapy Research \& Development LLC; Johnson \& Johnson Pharmaceutical Research \& Development LLC; Medpace Inc; Merck \& Co Inc; Meso Scale Diagnostics LLC; NeuroRx Research; Novartis Pharmaceuticals Corporation; Pfizer Inc; Piramal Imaging; Servier; Synarc Inc; and Takeda Pharmaceutical Company. The Canadian Institutes of Health Research is providing funds to support ADNI clinical sites in Canada. Private sector contributions are facilitated by the Foundation for the National Institutes of Health (\url{http://www.fnih.org}). The grantee organization is the Northern California Institute for Research and Education, and the study is coordinated by the Alzheimer’s Disease Cooperative Study at the University of California, San Diego. ADNI data are disseminated by the Laboratory for Neuro Imaging at the University of Southern California. Samples from the National Cell Repository for AD (NCRAD), which receives government support under a cooperative agreement grant (U24 AG21886) awarded by the National Institute on Aging (AIG), were used in this study. Funding for the WGS was provided by the Alzheimer’s Association and the Brin Wojcicki Foundation.
\bibliography{references}

\begin{thebibliography}{44}
\providecommand{\natexlab}[1]{#1}
\providecommand{\url}[1]{\texttt{#1}}
\expandafter\ifx\csname urlstyle\endcsname\relax
  \providecommand{\doi}[1]{doi: #1}\else
  \providecommand{\doi}{doi: \begingroup \urlstyle{rm}\Url}\fi

\bibitem[Arbel et~al.(2018)Arbel, Sutherland, Bi{\'n}kowski, and
  Gretton]{arbel2018gradient}
Michael Arbel, Dougal Sutherland, Miko{\l}aj Bi{\'n}kowski, and Arthur Gretton.
\newblock On gradient regularizers for mmd gans.
\newblock In \emph{Advances in Neural Information Processing Systems}, pages
  6700--6710, 2018.

\bibitem[Bausch(2013)]{bausch2013efficient}
Johannes Bausch.
\newblock On the efficient calculation of a linear combination of chi-square
  random variables with an application in counting string vacua.
\newblock \emph{Journal of Physics A: Mathematical and Theoretical},
  46\penalty0 (50):\penalty0 505202, 2013.

\bibitem[Bentkus(2005)]{bentkus2005lyapunov}
Vidmantas Bentkus.
\newblock A lyapunov-type bound in rd.
\newblock \emph{Theory of Probability \& Its Applications}, 49\penalty0
  (2):\penalty0 311--323, 2005.

\bibitem[Bi{\'n}kowski et~al.(2018)Bi{\'n}kowski, Sutherland, Arbel, and
  Gretton]{binkowski2018demystifying}
Miko{\l}aj Bi{\'n}kowski, Dougal~J Sutherland, Michael Arbel, and Arthur
  Gretton.
\newblock Demystifying mmd gans.
\newblock \emph{arXiv preprint arXiv:1801.01401}, 2018.

\bibitem[Chwialkowski et~al.(2015)Chwialkowski, Ramdas, Sejdinovic, and
  Gretton]{chwialkowski2015fast}
Kacper~P Chwialkowski, Aaditya Ramdas, Dino Sejdinovic, and Arthur Gretton.
\newblock Fast two-sample testing with analytic representations of probability
  measures.
\newblock In \emph{Advances in Neural Information Processing Systems}, pages
  1981--1989, 2015.

\bibitem[Corder et~al.(1993)Corder, Saunders, Strittmatter, Schmechel, Gaskell,
  Small, Roses, Haines, and Pericak-Vance]{corder1993apoe}
EH~Corder, AM~Saunders, WJ~Strittmatter, DE~Schmechel, PC~Gaskell, GW~Small,
  AD~Roses, JL~Haines, and MA~Pericak-Vance.
\newblock Gene dose of apolipoprotein e type 4 allele and the risk of
  alzheimer's disease in late onset families.
\newblock \emph{Science}, 261\penalty0 (5):\penalty0 921--923, 1993.

\bibitem[Dai et~al.(2017)Dai, Dai, Qu, Li, and Das]{dai2017very}
Wei Dai, Chia Dai, Shuhui Qu, Juncheng Li, and Samarjit Das.
\newblock Very deep convolutional neural networks for raw waveforms.
\newblock In \emph{2017 IEEE International Conference on Acoustics, Speech and
  Signal Processing (ICASSP)}, pages 421--425. IEEE, 2017.

\bibitem[Devroye et~al.(2013)Devroye, Gy{\"o}rfi, and
  Lugosi]{devroye2013probabilistic}
Luc Devroye, L{\'a}szl{\'o} Gy{\"o}rfi, and G{\'a}bor Lugosi.
\newblock \emph{A probabilistic theory of pattern recognition}, volume~31.
\newblock Springer Science \& Business Media, 2013.

\bibitem[Durrett(2019)]{durrett2019probability}
Rick Durrett.
\newblock \emph{Probability: theory and examples}, volume~49.
\newblock Cambridge university press, 2019.

\bibitem[Ernst et~al.(2004)]{ernst2004permutation}
Michael~D Ernst et~al.
\newblock Permutation methods: a basis for exact inference.
\newblock \emph{Statistical Science}, 19\penalty0 (4):\penalty0 676--685, 2004.

\bibitem[Eypasch et~al.(1995)Eypasch, Lefering, Kum, and
  Troidl]{eypasch1995probability}
Ernst Eypasch, Rolf Lefering, CK~Kum, and Hans Troidl.
\newblock Probability of adverse events that have not yet occurred: a
  statistical reminder.
\newblock \emph{Bmj}, 311\penalty0 (7005):\penalty0 619--620, 1995.

\bibitem[Friedman(2003)]{friedman2003multivariate}
Jerome Friedman.
\newblock On multivariate goodness-of-fit and two-sample testing.
\newblock In \emph{Statistical Problems in Particle Physics, Astrophysics, and
  Cosmology}, page 311, 2003.

\bibitem[Golowich et~al.(2017)Golowich, Rakhlin, and Shamir]{golowich2017size}
Noah Golowich, Alexander Rakhlin, and Ohad Shamir.
\newblock Size-independent sample complexity of neural networks.
\newblock \emph{arXiv preprint arXiv:1712.06541}, 2017.

\bibitem[Gramatik(2014)]{gramatiktaor}
Gramatik.
\newblock The age of reason.
\newblock \url{http://dl.lowtempmusic.com/Gramatik-TAOR.zip}, 2014.
\newblock [Online; accessed May/23/2019].

\bibitem[Gretton et~al.(2012{\natexlab{a}})Gretton, Borgwardt, Rasch,
  Sch{\"o}lkopf, and Smola]{gretton2012kernel}
Arthur Gretton, Karsten~M Borgwardt, Malte~J Rasch, Bernhard Sch{\"o}lkopf, and
  Alexander Smola.
\newblock A kernel two-sample test.
\newblock \emph{Journal of Machine Learning Research}, 13\penalty0
  (Mar):\penalty0 723--773, 2012{\natexlab{a}}.

\bibitem[Gretton et~al.(2012{\natexlab{b}})Gretton, Sejdinovic, Strathmann,
  Balakrishnan, Pontil, Fukumizu, and Sriperumbudur]{gretton2012optimal}
Arthur Gretton, Dino Sejdinovic, Heiko Strathmann, Sivaraman Balakrishnan,
  Massimiliano Pontil, Kenji Fukumizu, and Bharath~K Sriperumbudur.
\newblock Optimal kernel choice for large-scale two-sample tests.
\newblock In \emph{Advances in neural information processing systems}, pages
  1205--1213, 2012{\natexlab{b}}.

\bibitem[Hanin(2017)]{hanin2017universal}
Boris Hanin.
\newblock Universal function approximation by deep neural nets with bounded
  width and relu activations.
\newblock \emph{arXiv preprint arXiv:1708.02691}, 2017.

\bibitem[Harchaoui et~al.(2008)Harchaoui, Bach, and
  Moulines]{harchaoui2008testing}
Za{\"\i}d Harchaoui, Francis~R Bach, and Èric Moulines.
\newblock Testing for homogeneity with kernel fisher discriminant analysis.
\newblock In \emph{Advances in Neural Information Processing Systems}, pages
  609--616, 2008.

\bibitem[He et~al.(2016)He, Zhang, Ren, and Sun]{he2016deep}
Kaiming He, Xiangyu Zhang, Shaoqing Ren, and Jian Sun.
\newblock Deep residual learning for image recognition.
\newblock In \emph{Proceedings of the IEEE conference on computer vision and
  pattern recognition}, pages 770--778, 2016.

\bibitem[Holmes et~al.(2015)Holmes, Hollinshead, O’Keefe, Petrov, Fariello,
  Wald, Fischl, Rosen, Mair, Roffman, et~al.]{holmes2015brain}
Avram~J Holmes, Marisa~O Hollinshead, Timothy~M O’Keefe, Victor~I Petrov,
  Gabriele~R Fariello, Lawrence~L Wald, Bruce Fischl, Bruce~R Rosen, Ross~W
  Mair, Joshua~L Roffman, et~al.
\newblock Brain genomics superstruct project initial data release with
  structural, functional, and behavioral measures.
\newblock \emph{Scientific data}, 2:\penalty0 150031, 2015.

\bibitem[Ioffe and Szegedy(2015)]{ioffe2015batch}
Sergey Ioffe and Christian Szegedy.
\newblock Batch normalization: Accelerating deep network training by reducing
  internal covariate shift.
\newblock \emph{arXiv preprint arXiv:1502.03167}, 2015.

\bibitem[Jitkrittum et~al.(2016)Jitkrittum, Szab{\'o}, Chwialkowski, and
  Gretton]{jitkrittum2016interpretable}
Wittawat Jitkrittum, Zolt{\'a}n Szab{\'o}, Kacper~P Chwialkowski, and Arthur
  Gretton.
\newblock Interpretable distribution features with maximum testing power.
\newblock In \emph{Advances in Neural Information Processing Systems}, pages
  181--189, 2016.

\bibitem[Jitkrittum et~al.(2018)Jitkrittum, Kanagawa, Sangkloy, Hays,
  Sch{\"o}lkopf, and Gretton]{jitkrittum2018informative}
Wittawat Jitkrittum, Heishiro Kanagawa, Patsorn Sangkloy, James Hays, Bernhard
  Sch{\"o}lkopf, and Arthur Gretton.
\newblock Informative features for model comparison.
\newblock In \emph{Advances in Neural Information Processing Systems}, pages
  808--819, 2018.

\bibitem[Khosla et~al.(2011)Khosla, Jayadevaprakash, Yao, and
  Fei-Fei]{KhoslaYaoJayadevaprakashFeiFei_FGVC2011}
Aditya Khosla, Nityananda Jayadevaprakash, Bangpeng Yao, and Li~Fei-Fei.
\newblock Novel dataset for fine-grained image categorization.
\newblock In \emph{First Workshop on Fine-Grained Visual Categorization, IEEE
  Conference on Computer Vision and Pattern Recognition}, Colorado Springs, CO,
  June 2011.

\bibitem[Kim et~al.(2016)Kim, Ramdas, Singh, and
  Wasserman]{kim2016classification}
Ilmun Kim, Aaditya Ramdas, Aarti Singh, and Larry Wasserman.
\newblock Classification accuracy as a proxy for two sample testing.
\newblock \emph{arXiv preprint arXiv:1602.02210}, 2016.

\bibitem[Kingma and Ba(2014)]{kingma2014adam}
Diederik~P Kingma and Jimmy Ba.
\newblock Adam: A method for stochastic optimization.
\newblock \emph{arXiv preprint arXiv:1412.6980}, 2014.

\bibitem[Lehmann and Romano(2006)]{lehmann2006testing}
Erich~L Lehmann and Joseph~P Romano.
\newblock \emph{Testing statistical hypotheses}.
\newblock Springer Science \& Business Media, 2006.

\bibitem[Li et~al.(2017)Li, Chang, Cheng, Yang, and P{\'o}czos]{li2017mmd}
Chun-Liang Li, Wei-Cheng Chang, Yu~Cheng, Yiming Yang, and Barnab{\'a}s
  P{\'o}czos.
\newblock Mmd gan: Towards deeper understanding of moment matching network.
\newblock In \emph{Advances in Neural Information Processing Systems}, pages
  2203--2213, 2017.

\bibitem[Lopez-Paz and Oquab(2016)]{lopez2016revisiting}
David Lopez-Paz and Maxime Oquab.
\newblock Revisiting classifier two-sample tests.
\newblock \emph{arXiv preprint arXiv:1610.06545}, 2016.

\bibitem[Lu et~al.(2015)Lu, Behbood, Hao, Zuo, Xue, and Zhang]{lu2015transfer}
Jie Lu, Vahid Behbood, Peng Hao, Hua Zuo, Shan Xue, and Guangquan Zhang.
\newblock Transfer learning using computational intelligence: a survey.
\newblock \emph{Knowledge-Based Systems}, 80:\penalty0 14--23, 2015.

\bibitem[Lundqvist et~al.(1998)Lundqvist, Flykt, and
  {\"O}hman]{lundqvist1998karolinska}
Daniel Lundqvist, Anders Flykt, and Arne {\"O}hman.
\newblock The karolinska directed emotional faces (kdef).
\newblock \emph{CD ROM from Department of Clinical Neuroscience, Psychology
  section, Karolinska Institutet}, 91:\penalty0 630, 1998.

\bibitem[Maji et~al.(2013)Maji, Kannala, Rahtu, Blaschko, and
  Vedaldi]{maji13fine-grained}
S.~Maji, J.~Kannala, E.~Rahtu, M.~Blaschko, and A.~Vedaldi.
\newblock Fine-grained visual classification of aircraft.
\newblock Technical report, 2013.

\bibitem[Mieth et~al.(2016)Mieth, Kloft, Rodr{\'\i}guez, Sonnenburg, Vobruba,
  Morcillo-Su{\'a}rez, Farr{\'e}, Marigorta, Fehr, Dickhaus,
  et~al.]{mieth2016combining}
Bettina Mieth, Marius Kloft, Juan~Antonio Rodr{\'\i}guez, S{\"o}ren Sonnenburg,
  Robin Vobruba, Carlos Morcillo-Su{\'a}rez, Xavier Farr{\'e}, Urko~M
  Marigorta, Ernst Fehr, Thorsten Dickhaus, et~al.
\newblock Combining multiple hypothesis testing with machine learning increases
  the statistical power of genome-wide association studies.
\newblock \emph{Scientific reports}, 6:\penalty0 36671, 2016.

\bibitem[Mohri et~al.(2018)Mohri, Rostamizadeh, and
  Talwalkar]{mohri2018foundations}
Mehryar Mohri, Afshin Rostamizadeh, and Ameet Talwalkar.
\newblock \emph{Foundations of machine learning}.
\newblock MIT press, 2018.

\bibitem[Mueller et~al.(2005)Mueller, Weiner, Thal, Petersen, Jack, Jagust,
  Trojanowski, Toga, and Beckett]{mueller2005alzheimer}
Susanne~G Mueller, Michael~W Weiner, Leon~J Thal, Ronald~C Petersen, Clifford
  Jack, William Jagust, John~Q Trojanowski, Arthur~W Toga, and Laurel Beckett.
\newblock The alzheimer's disease neuroimaging initiative.
\newblock \emph{Neuroimaging Clinics}, 15\penalty0 (4):\penalty0 869--877,
  2005.

\bibitem[Neyman and Pearson(1933)]{neyman1933problem}
J~Neyman and ES~Pearson.
\newblock On the problem of the most efficient tests of statistical hypotheses.
\newblock \emph{Philosophical Transactions of the Royal Society of London.
  Series A, Containing Papers of a Mathematical or Physical Character},
  231:\penalty0 289--337, 1933.

\bibitem[Paszke et~al.(2017)Paszke, Gross, Chintala, Chanan, Yang, DeVito, Lin,
  Desmaison, Antiga, and Lerer]{paszke2017automatic}
Adam Paszke, Sam Gross, Soumith Chintala, Gregory Chanan, Edward Yang, Zachary
  DeVito, Zeming Lin, Alban Desmaison, Luca Antiga, and Adam Lerer.
\newblock Automatic differentiation in pytorch.
\newblock In \emph{NIPS-W}, 2017.

\bibitem[Prechelt(1998)]{prechelt1998early}
Lutz Prechelt.
\newblock Early stopping-but when?
\newblock In \emph{Neural Networks: Tricks of the trade}, pages 55--69.
  Springer, 1998.

\bibitem[Russakovsky et~al.(2015)Russakovsky, Deng, Su, Krause, Satheesh, Ma,
  Huang, Karpathy, Khosla, Bernstein, Berg, and Fei-Fei]{ILSVRC15}
Olga Russakovsky, Jia Deng, Hao Su, Jonathan Krause, Sanjeev Satheesh, Sean Ma,
  Zhiheng Huang, Andrej Karpathy, Aditya Khosla, Michael Bernstein,
  Alexander~C. Berg, and Li~Fei-Fei.
\newblock {ImageNet Large Scale Visual Recognition Challenge}.
\newblock \emph{International Journal of Computer Vision (IJCV)}, 115\penalty0
  (3):\penalty0 211--252, 2015.
\newblock \doi{10.1007/s11263-015-0816-y}.

\bibitem[Steinwart and Christmann(2008)]{steinwart2008support}
Ingo Steinwart and Andreas Christmann.
\newblock \emph{Support vector machines}.
\newblock Springer Science \& Business Media, 2008.

\bibitem[Wah et~al.(2011)Wah, Branson, Welinder, Perona, and
  Belongie]{WahCUB_200_2011}
C.~Wah, S.~Branson, P.~Welinder, P.~Perona, and S.~Belongie.
\newblock {The Caltech-UCSD Birds-200-2011 Dataset}.
\newblock Technical Report CNS-TR-2011-001, California Institute of Technology,
  2011.

\bibitem[Xu et~al.(2018)Xu, Huang, Yuan, Guo, Sun, Wu, and
  Weinberger]{xu2018empirical}
Qiantong Xu, Gao Huang, Yang Yuan, Chuan Guo, Yu~Sun, Felix Wu, and Kilian
  Weinberger.
\newblock An empirical study on evaluation metrics of generative adversarial
  networks.
\newblock \emph{arXiv preprint arXiv:1806.07755}, 2018.

\bibitem[Zhou et~al.(2016)Zhou, Ithapu, Ravi, Singh, Wahba, and
  Johnson]{zhou2016hypothesis}
Hao Zhou, Vamsi~K Ithapu, Sathya~Narayanan Ravi, Vikas Singh, Grace Wahba, and
  Sterling~C Johnson.
\newblock Hypothesis testing in unsupervised domain adaptation with
  applications in alzheimer's disease.
\newblock In \emph{Advances in neural information processing systems}, pages
  2496--2504, 2016.

\bibitem[Zhou and Guan(2018)]{zhou2018null}
Quan Zhou and Yongtao Guan.
\newblock On the null distribution of bayes factors in linear regression.
\newblock \emph{Journal of the American Statistical Association}, 113\penalty0
  (523):\penalty0 1362--1371, 2018.

\end{thebibliography}
\bibliographystyle{plainnat}

\clearpage

\appendix

\section{PROOF OF THEOREMS}
\subsection{Control of type-1 error rate}
\label{app-proofs-pvals}
\begin{proof}[Proof of Theorem~\ref{thm-pvals}]

    \emph{(i)}
    Under $p = q$, it holds that $\E[\phi(X_1)] = \E[\phi(Y_1)]$ and $\Sigma = \Cov(\phi(X_1)) = \Cov(\phi(Y_1))$.
    Then we have
    \begin{align*}
        & \sqrt{\frac{nm}{n+m}} D_{n,m}(\phi)
        \\
        & \, = \sqrt{\frac{nm}{n+m}} \left( \frac{1}{n}\sum_{i=1}^n \phi(X_i) - \E[\phi(X_i)] \right.
        \\
        & \,\, \left. - \frac{1}{m} \sum_{i=1}^m \phi(Y_i) - \E[\phi(Y_i)] \right)
        \\
        & = \sqrt{\frac{m}{n+m}} \frac{1}{\sqrt{n}} \sum_{i=1}^n \left( \phi(X_i) - \E[\phi(X_i)] \right)
        \\
        & - \sqrt{\frac{n}{n+m}} \frac{1}{\sqrt{m}} \sum_{i=1}^m \left(\phi(Y_i) - \E[\phi(Y_i)]\right)
    \end{align*}
    Then the first term in the last expression converges in distribution against $\NN(0, r \Sigma)$ and the second term converges in distribution against $\NN(0, (1-r)\Sigma)$ by a multivariate Central Limit Theorem (note that $\phi(X_1)$ lies within $[-1,1]^H$ and hence all moments are finite).
    Since all $X_i$ and $Y_j$ are jointly independent, the limiting distributions are also independent, hence the whole term converges against $\NN(0, r\Sigma) - \NN(0, (1-r)\Sigma) = \NN(0, \Sigma)$.
    
    \emph{(ii)}
    By \emph{(i)} and the continuous mapping theorem, $S_{n,m}(\phi, \XX_n, \YY_m) \dto ||\zeta||^2$, where $\zeta \sim \NN(0, \Sigma)$.
    Since $\Sigma$ is positive semi-definite, there exist an orthogonal matrix $Q$ and a diagonal matrix $L = \mbox{diag}(\lambda_1, \ldots, \lambda_d)$ such that $\Sigma = QLQ^\top$.
    Then we have
    \begin{align*}
        ||Q\zeta||^2 = \zeta^\top Q^\top Q \zeta = \zeta^\top \zeta = ||\zeta||^2,
    \end{align*}
    and $Q\zeta \sim \NN(0, L)$, hence the claim.
    
    \emph{(iii)}
    By the weak law of large numbers, $\hat{\Sigma}_{n,m} \pto \Sigma$, and hence by \emph{(i)} and Slutsky's Theorem
    \begin{align*}
        \sqrt{\frac{nm}{n+m}}\hat{\Sigma}_{n,m}^{-\frac{1}{2}}D_{n,m}(\phi) \dto \NN(0, I).
    \end{align*}
    The rest follows again by the continuous mapping theorem.
    
\end{proof}

\subsection{Proof of Consistency}
\label{app-proofs-consistency}
Before we begin the proof we start with some auxiliary definitions and preliminary results.

As in Section~\ref{sec:testingprocedure} we can use the regression framework with $(Z_i, t_i)_i \subset \R^d \times \{-1, 1\}$.
Then $Z_i | t_i = 1 \sim p$, $Z_i | t_i = -1 \sim q$ and similarly for $(Z_i', t_i')$ and all jointly independent.
As we assume $\Pr(t=1) = \Pr(t=-1) = \frac{1}{2}$, the distribution of $(Z, t)$ is fully determined by specifying $p$ and $q$ and hence we write for the expected value e.g. $\E_{p,q}[f(Z,t)]$ for some function $f$.

We define the loss function
\begin{align*}
    L(t, \hat{t}) := 1 - t\hat{t} \in [0, 2]
\end{align*}
    with corresponding empirical and expected risks
\begin{align*}
    R_N'(\psi) & = \frac{1}{N} \sum_{i=1}^N 1 - t_i'\psi(Z_i'),
    \\
    R'(\psi) & = 1 - \E_{p', q'}[t\psi(Z)].
\end{align*}
The Bayes risk under the transfer task will be denoted as $R'^* = \inf_{f\in \MM} R'(f) = 1- \epsilon'$ where $\MM$ is class of all Borel-measurable functions from $\R^d \to [-1, 1]$

Selecting $\phi_N$ is equivalent to (inexact) empirical risk minimization over $\GG_N := \{ w^\top \phi | \phi \in \TF_N, ||w||\leq 1\}$, i.e.
\begin{align*}
    R_N'(\psi_N) \leq \min_{\psi\in\GG_N} R_N'( \psi) + \eta
\end{align*}
where $\psi_N = w^\top \phi_N \in \GG_N$ for some $||w_N||\leq 1$.

The following Lemma is based on Theorem 1 in \citep{golowich2017size} and we will need it to bound the complexity of the neural network function class $\GG_N$.
\begin{lemma}\label{lemma-consistency-rademacherbound}
    Let the data be a.s. be bounded by some $B > 0$ and
\begin{align*}
    \GG & := \left\{ w^\top \tanh \circ  W_{D'-1} \circ \sigma \circ \ldots \circ \sigma \circ W_1 : \R^d \to \R \right| 
    \\
    & W_1 \in \R^{H\times d}, W_j \in \R^{H\times H} \mbox{ for } j=2, \ldots, D'-1, 
    \\
    & w \in \R^{H} \mbox{ with } ||w||\leq 1,
    \\
    & \left. \prod_{j=1}^{D'-1}||W_j||_{Fro} \leq \beta, D'\leq D \right\}
\end{align*}
    
    Then, the empirical Rademacher complexity of $\GG$ can be bounded as:
    \begin{align*}
        \hat{\RR}_N(\GG) \leq \frac{B (d+1)(\sqrt{2\log(2)(D -1)} + 1)\beta}{\sqrt{N}}.
\end{align*}
\end{lemma}
\begin{proof}[Proof of Lemma~\ref{lemma-consistency-rademacherbound}]
We define auxiliary function classes
\begin{align*}
    \GG_{D-1}^s & := \left\{ W_{D'-1} \circ \sigma \circ \ldots \ldots \sigma \circ W_1 :\R^d \to \R^s \right|
    \\
    & W_1 \in \R^{H\times d}, W_j \in \R^{H\times H} \mbox{ for } j=2, \ldots, D'-2,
    \\
    & W_{D'-1} \in \R^{s\times H},
    \\
    & \left. \prod_{j=1}^{D'-1}||W_j||_{Fro} \leq \beta, D' \leq D \right\}
\end{align*}
    for $s \in \{1, H\}$.

    Then we can rewrite $\GG$ as
    \begin{align*}
        \GG & = \{ \sum_{j=1}^H w_j \tanh\circ\phi_j | ||w|| \leq 1, \phi \in \GG_{D-1}^H \}
        \\
        & \subset \{ \sum_{j=1}^H w_j \tanh\circ\phi_j | ||w|| \leq 1, \phi_j \in \GG_{D-1}^1 \}
        \\
        & \subset \sum_{j=1}^H \{ w \tanh\circ\phi | |w| \leq 1, \phi \in \GG_{D-1}^1 \}.
    \end{align*}
    Therefore we can bound the Rademacher complexity as
    \begin{align*}
        \hat{\RR}_N(\GG) & \leq H \hat{\RR}_n(\{w \tanh\circ\phi | |w| \leq 1, \phi \in \GG_{D-1}^1 \})
        \\
        & \leq H \hat{\RR}_n(\{\tanh \circ\phi | \phi \in \GG_{D-1}^1 \})    
        \\
        & \leq H \hat{\RR}_n(\GG_{D-1}^1)
    \end{align*}
    by standard learning theory arguments.
    For $\GG_{D-1}^1$, we use the Rademacher bound found in \citep{golowich2017size} Theorem 1 (we cannot use the Theorem directly on $\GG$ since $\tanh$ is not positive homogeneous):
    \begin{align*}
        \hat{\RR}_N(\GG_{D-1}^1) \leq \frac{B (\sqrt{2\log(2)(D-1)} + 1) \beta}{\sqrt{N}}.
    \end{align*}
    The original Theorem 1 in \citep{golowich2017size} holds for depth $D-1$ networks, but we allowed networks of lower depth.
    However, one can fill up the networks to depth $D-1$ with identity weight matrices and identity activation functions; inspection of the proof of the Theorem then shows that the claim still holds.
    
    Since $H = d+1$, the claim follows.
\end{proof}

With these preliminary notions set up, we can proceed with the actual proof of consistency.
\begin{proof}[Proof of Theorem~\ref{thm-consistency}]
    We intend to show that $R_{n,m}(\psi_N)$ is asymptotically strictly smaller than 1; the divergence of the test statistic then follows easily.
    We will proceed in 5 steps.
    First, we split $R(\psi_N) - R'^*$ into transfer error, estimation error (of the transfer task) and approximation error (of the transfer task).
    Second, we show that the approximation error converges to zero (due to a Universal Approximation Theorem for deep networks); third we show that the estimation error is asymptotically bounded by $\eta$, using a learning theory bound on the Rademacher complexity of the neural network function class.
    This together implies that $R(\psi_N)$ and $R'^*$ are $(\delta + \eta)$-close asymptotically.
    Fourth, we show that the $R_{n,m}(\psi_N) - R(\psi_N) \pto 0$ and from this we finally deduce that the test statistics diverge to $+\infty$.

    \paragraph{1. Splitting the terms}
    We have
    \begin{align*}
        R(\psi_N) - R'^* = \left[R(\psi_N) - R'(\psi_N)\right] + \left[R'(\psi_N) - R'^* \right].
    \end{align*}
    The first term is bounded as follows:
    \begin{align*}
        |R(\psi_N) & - R'(\psi_N) |
        \\
        & = \left| \E_{p,q}[t\psi_N(Z)] - \E_{p', q'}[t\psi_N(Z)] \right|
        \\
        & \leq \frac{||\psi_N||_{\infty}}{2} (||p - p'||_{L_1(\mu)} + ||q - q'||_{L_1(\mu)})
        \\
        & \leq \delta,
    \end{align*}
    due to boundedness of $\psi_N$ and requirement \emph{(ii)}.
    
    The second term can again be split:
    \begin{align*}
        R'&(\psi_N)  - R'^*
        \\
        & = \left[ R'(\psi_N) - \min_{\psi\in\GG_N}R'(\psi)\right] + \left[\min_{\psi\in\GG_N}R'(\psi) - R'^* \right].
    \end{align*}
    
    \paragraph{2. Convergence of $\min_{\GG_N} R'(\psi) - R'^*$:}
    Let $\hat{\mu}$ be the Borel measure of $Z$ (not conditioned on $t$), i.e. $\Pr(Z\in A) = \hat{\mu}(A)$ for any $A\subset\R^d$ Borel.
    Following a similar argument as Lemma 30.2 in \citep{devroye2013probabilistic} then yields the following.
    For any fixed $\epsilon > 0$, select a measurable function $h:\R^d \to [-1, 1]$ such that $|R(h) - R^*| \leq \frac{\epsilon}{4}$, and a compact set $K \subset \R^d$ with $\hat{\mu}(K) \geq 1 - \frac{\epsilon}{8}$.
    Then, since compact-support continuous functions are dense in $L_1(\mu)$, there exists a continuous function $f:\R^d\to[-1,1]$ with
    \begin{align*}
        \E[|f(Z) - h(Z)| \mathbb{1}_{Z\in K}] \leq \frac{\epsilon}{4}.
    \end{align*}
    From the universal approximation theorem for deep ReLU networks in \citep{hanin2017universal}, there exists $N_0 \geq 1$ such that for all $N\geq N_0$ we can find a $\psi \in \GG_N$ with
    \begin{align*}
        \sup_{z\in K} |f(z) - \psi(z)|\leq \frac{\epsilon}{4}.
    \end{align*}
    Note that the Theorem in \citep{hanin2017universal} holds for ReLU-networks, but since $\tanh$ is invertible one can apply the universal approximation theorem on the first node in the last hidden layer, select the $w_N = [1, 0, \ldots, 0]^\top$  and still get the universal approximation property.
    
    Combining these yields, for $N$ large enough,
    \begin{align*}
        \min_{\psi\in\GG_N}& R'(\psi) - R'^* \leq R'(\psi) - R'^*
        \\
        = & \E_{p', q'}[-t\psi(Z) + th(Z)] + R'(m) - R'^*
        \\
        \leq & \E[|\psi(Z) - h(Z)|\mathbb{1}_{Z\in K}] + 2 \hat{\mu}(K^c) + \frac{\epsilon}{4}
        \\
        \leq & \E[|\psi(Z) - f(Z)|\mathbb{1}_{Z\in K}] 
        \\
        & + \E[|f(Z) - h(Z)|\mathbb{1}_{Z\in K}] + \frac{\epsilon}{2}
        \\
        \leq & \epsilon.
    \end{align*}
    Then, since $\epsilon>0$ was arbitrary, and $\min_{\psi\in\GG_N}R'(\psi) \ge  R'^*$ we get $\min_{\GG_N} R'(\psi) \to R'^*$ as $N\to \infty$.
    
    \paragraph{3. Asymptotic closeness of $R'(\psi_N)$ and $\min_{\GG_N}R'(\psi)$:}
    We can first bound by standard arguments:
    \begin{align*}
        R'&(\psi_N) - \min_{\psi\in\GG_N}R'(\psi) \\
        & = \left[R'(\psi_N) - R_N'(\psi_N)\right] + \left[R_N'(\psi_N) - \min_{\psi\in\GG_N}R(\psi)\right]
        \\
        & \leq \max_{\psi\in\GG_N}\left[ \left|R'(\psi) - R_N'(\psi)\right| \right]
        \\
        & \,\,\, + \left[\min_{\psi\in\GG_N}R_N'(\psi) + \eta - \min_{\psi\in\GG_N}R(\psi)\right]
        \\
        & \leq 2 \max_{\psi\in\GG_N}\left|R'(\psi) - R_N'(\psi)\right| + \eta
        \\
        & = 2 \sup_{h\in\HH_N} \left| \E[h(Z', t')] - \frac{1}{N}\sum_{i=1}^N g(Z_i', t_i')\right| + \eta
    \end{align*}
    as $R_N'(\psi_N) \leq \min R_N'(\psi) + \eta$, where we define $\HH_N := \left\{ (z, t) \mapsto L(\psi(Z), t) | \psi \in \GG_N \right\}$ as the conjunction of neural networks with the loss function.
    The first term can be bound with high probability by two-sided Rademacher inequalities:
    \begin{align*}
        \Pr& \left( \sup_{h\in\HH_N} \left| \E[h(Z', t')]  - \frac{1}{N}\sum_{i=1}^N g(Z_i', t_i')\right| \right.
        \\
        & \quad\quad\quad\quad\quad\quad \left. \leq 2 \hat{\RR}_N(\GG_N) + 6 \sqrt{\frac{\log(4/\zeta)}{2N}} \right) 
        \\
        & \geq 1 - \zeta
    \end{align*}
    for any $\zeta > 0$.
    This complexity bound follows from Theorem 11.3 in \citep{mohri2018foundations} if we insert the function class $\hat{\HH} := \HH_N \cup 2 - \HH_N$ by noting that the loss function is 1-Lipschitz in both its arguments non-negative and bounded from above by 2.
    
    Setting $\epsilon := 2\hat{\RR}_N(\GG_N) + 6\sqrt{\frac{\log(4/\zeta)}{2N}}$ then yields
    \begin{align}
    \label{eq-consistency-concentration}
         \Pr&\left( \sup_{h\in\HH_N} \left| \E[h(Z', t')]  - \frac{1}{N}\sum_{i=1}^N g(Z_i', t_i')\right| > \epsilon \right) \nonumber
         \\
         & \leq 4 \exp\left(-\frac{N(\epsilon - 2\hat{\RR}_N(\GG_N))^2}{18} \right).
    \end{align}
    But Lemma~\ref{lemma-consistency-rademacherbound} bounds the Rademacher complexity as
    \begin{align}
    \label{eq-rademacher-bound}
        \hat{\RR}_N(\GG_N) & \leq \frac{B (d+1)\left(\sqrt{2\log(2)(D_N -1)} + 1\right)\beta_N}{\sqrt{N}} \nonumber
        \\
        & \leq C \frac{\sqrt{D_N}\beta_N}{\sqrt{N}}
    \end{align}
    for some $C>0$ and $D_N$ large enough.
    Then
    \begin{align*}
        N (\epsilon - 2\hat{\RR}_N(\GG_N))^2 & \geq N\epsilon^2  - 4 N \hat{\RR}_N(\GG_N)
        \\
        & \geq N \epsilon^2 - 4 C \sqrt{N}\sqrt{D_N}\beta_N, 
    \end{align*}
    and the last term diverges to $\infty$ if $\frac{\beta_N^2D_N}{N}\to 0$.
    Hence, the right-hand side in equation~\eqref{eq-consistency-concentration} converges to 0.
    
    This shows that 
    \begin{align*}
        \Pr(R'(\psi_N) - \min_{\psi\in\GG_N}R'(\psi) \leq \epsilon + \eta) \to 1
    \end{align*}
    for any $\epsilon > 0$, i.e. $R'(\psi_N)$ is asymptotically $\eta$-close to $\min_{\GG_N}R'(\psi)$ (in probability). 
    
    \paragraph{4. \boldmath$R_{n,m}(\psi_N)-R(\psi_N) \pto 0$}
    Next we need to show that the empirical risk (over $(Z, t)$, not $(Z', t')$) also is asymptotically smaller than 1.
    
    We look at $\xi_{N,i} := t_i \psi_N(Z_i)$, which is a triangular array of random variables on $[-1,1]$.
    We will use a weak law of large numbers for triangular arrays, see Theorem 2.2.11 in \citep{durrett2019probability}.
    Both requirements in the Theorem are satisfied since $\xi_{N,i}$ is bounded, and hence we get
    \begin{align*}
        \frac{1}{N} \sum_{i=1}^N & t_i\psi_N(Z_i) - \E[t\psi_N(Z)] 
        \\
        & = \frac{\sum_{i=1}^N \xi_{N,i} - N\E[\xi_{N,i}]}{N} \pto 0,
    \end{align*}
    or equivalently $R_{n,m}(\psi_N) - R(\psi_N) \pto 0$.
    
    But as shown above, $R(\psi_N)$ is $\delta$-close to $R'(\psi_N)$ and $R'(\psi_N)$ is asymptotically $\eta$-close to $R'^*$; hence we get
    \begin{align*}
        \Pr(R(\psi_N) - R'^* \leq \epsilon + \delta + \eta) \to 1
    \end{align*}
    for any $\epsilon > 0$, and therefore
    \begin{align*}
        \Pr(R_{n,m}(\psi_N) - R'^* \leq \epsilon + \delta + \eta) \to 1
    \end{align*}

    \paragraph{5. Divergence of test statistics}
    Define $M_N = 1 - R_{n,m}(\psi_N)$, then
    \begin{align*}
        \Pr(M_N \geq \epsilon^* - \delta - \eta - \epsilon) \to 1
    \end{align*}
    for any $\epsilon > 0$.
    Since $\delta + \eta < \epsilon^*$, we then have for any $r > 0:$
    \begin{align*}
        \Pr\left(\sqrt{\frac{nN}{m}} M_N > r\right) = \Pr\left(M_N > \sqrt{\frac{m}{nN}}r\right) \to 1,
    \end{align*}
    i.e. $M_N \pto +\infty$, since $\sqrt{\frac{m}{nN}} \to 0.$
    
    Next, define
    \begin{align*}
        \hat{S}_{n,m} = \frac{nm}{n+m} \left( \frac{1}{n}\sum_{i=1}^n \psi_N(X_i) - \frac{1}{m}\sum_{i=1}^m \psi_N(Y_i)\right),
    \end{align*}
    i.e. the version of $S_{n,m}$ where the last layer is still selected on the training data instead of the test data.
    Then it holds that
    \begin{align*}
        & \left| \sqrt{\frac{m}{n(m+n)}}\hat{S}_{n,m} -  M_N \right|
        \\
        = &  \left| \frac{1}{m+n}\left(\frac{m}{n}\sum_{i=1}^n \psi_N(X_i) - \sum_{i=1}^m \psi_N(Y_i) \right) \right.
        \\
        & \quad\quad\quad \left. - \frac{1}{m+n}\sum_{i=1}^{m+n}t_i\psi_N(Z_i)\right|
        \\
        & = \left| \frac{1}{m+n}\sum_{i=1}^n \psi_N(X_i)\right| \left|\frac{m}{n} -1 \right|
        \\
        & \leq \left|\frac{m}{n} -1 \right|  \to 0,
    \end{align*}
    since all $|\psi_N(X_i)| \leq 1$ and $\frac{m}{n} \to 1$.
    
    Hence, we also get
    \begin{align*}
        \Pr(\hat{S}_{n,m} > r) \to 1
    \end{align*}
    for any $r > 0$.
    But
    \begin{align*}
        S_{n,m} & = \frac{nm}{n+m} \left|\left| \overline{\phi_N(\XX_n)} - \overline{\phi_N(\YY_m)}\right|\right|^2
        \\
        & = \frac{nm}{n+m} \sup_{||w||\leq 1} w^\top \left( \overline{\phi_N(\XX_n)} - \overline{\phi_N(\YY_m)} \right)
        \\
        & \geq \hat{S}_{n,m}
    \end{align*}
    
    For the DFDA test statistic, we have
    \begin{align*}
        T_{n,m} & = \frac{nm}{n+m}D_{n,m}^\top \hat{\Sigma}_{n,m}^{-1} D_{n,m}
        \\
        & \geq  \frac{nm}{n+m}||D_{n,m}||^2 \lambda_{\min}(\hat{\Sigma}_{n,m}^{-1})
        \\
        & = S_{n,m} \lambda_{\max}(\hat{\Sigma}_{n,m})^{-1}.
    \end{align*}
    $\lambda_{\max}(\hat{\Sigma}_{n,m})$ is always positive (due to the $\rho_{n,m}>0$ summand), and also bounded from above by some $C>0$ (due to the boundedness of all individual entries), therefore $T_{n,m} \geq C^{-1} S_{n,m}$.
    
    Hence we also have $S_{n,m}, T_{n,m} \pto +\infty$.
\end{proof}

\section{DISTRIBUTIONS WITH UNBOUNDED SUPPORT}
\label{sec:unbounded}
Considering the case where $p'$ and $q'$ have unbounded support, but requirements \emph{(ii)}, \emph{(iii)} and a variant of \emph{(i)} in Theorem~\ref{thm-consistency} are still satisfied, we can still prove a similar consistency result.

In particular, we can make $p'$ and $q'$ vary with $N$ by replacing them with truncated, bounded-support versions that converge towards the true densities slowly enough.
First, select $p_N'$ and $q_N'$ with support on $[-B_N, B_N]^d$ for some sequence $B_n \uparrow +\infty$, and $||p_N' - p'||_{L_1(\mu)} \to 0$ and $||q_N' - q'||_{L_1(\mu)} \to 0$.
Then there exists a $N_0 > 0$ such that for all $N \geq N_0$, requirements \emph{(i)}, \emph{(ii)} and \emph{(iii)} are satisfied for $p_N'$ and $q_N'$.
In practice these truncated variables can be achieved for example by rejection sampling from $p'$ and $q'$.

The only part in the proof of Theorem~\ref{thm-consistency} where we need the boundedness assumption on $p'$ and $q'$ is when bounding the Rademacher complexity of the class $\GG_N$ in equation~\ref{eq-rademacher-bound}.
The modified Rademacher bound now is
\begin{align*}
        \hat{\RR}_N(\GG_N) & \leq \frac{B_N (d+1)\left(\sqrt{2\log(2)(D_N -1)} + 1\right)\beta_N}{\sqrt{N}}
        \\
        & \leq C \frac{\sqrt{D_N}B_N\beta_N}{\sqrt{N}}.
\end{align*}
The requirement for the exponent in equation~\eqref{eq-consistency-concentration} to diverge then is
\begin{align*}
    \frac{B_N^2 \beta_N^2 D_N}{N} & \to 0 \mbox{ instead of }
    \\
    \frac{\beta_N^2 D_N}{N} & \to 0.
\end{align*}
The rest of the proof is as before.
We can summarize this as follows:

\begin{theorem}
    Let $p \neq q$, $n=n', m=m'$ with $\frac{n}{m}\to 1$, $N = n + m$, $R'^* = 1 - \epsilon'$ the Bayes error for the transfer task with $\epsilon' > 0$.
    Furthermore, let $||p_N' - p'||_{L_1(\mu)} \to 0$ and $||q_N' - q'||_{L_1(\mu)} \to 0$ for sequences of $\mu$-densities $(p_N')_N$ and $(q_N')_N$,
    
    Assume that the following holds:
    \begin{itemize}
        \item[(i)] $\frac{B_N^2\beta_N^2 D_N}{N} \to 0$, $B_N \to \infty$, $\beta_N \to \infty$ and $D_N \to \infty$ for $N\to\infty$,
        \item[(ii)] $||p - p'||_{L_1(\mu)} + ||q - q'||_{L_1(\mu)} < 2 \delta$,
        \item[(iii)] $0 \leq \delta + \eta < \epsilon'$, where $\eta\geq 0$ is the leniency parameter in training the network, and
        \item[(iv)] for each $N$, $p_N'$ and $q_N'$ have support on $[-B_N, B_N]^d$.
    \end{itemize}
    Then, as $N\to\infty$ both test test statistics $S(\phi_N, \XX_n, \YY_m)$ and $T(\phi_N, \XX_n, \YY_m)$ diverge in probability towards infinity, i.e. for any $r>0$
    \begin{align*}
        & \Pr\left(S(\phi_N, \XX_n, \YY_m) > r\right) \to 1 \mbox{ and }
        \\
        & \Pr\left(T(\phi_N, \XX_n, \YY_m) > r\right) \to 1.
    \end{align*}
\end{theorem}

\section{DIMENSIONALITY REDUCTION}
\label{app-pca}
In practice, we oftentimes first apply a PCA transformation on the data before computing the DFDA test statistic.
Since we fit the PCA on the test data itself, however, the observations are not independent anymore and Theorem~\ref{thm-pvals} is not directly applicable anymore.
As an unsupervised linear transformation, however, we can show via a Slutsky-type argument that the normal approximation is still valid.

\begin{theorem}
\label{thm-pca}
Let $(\xi_i)_i$ and $(\xi_i')_i$ be all jointly independent and identically distributed on $\R^d$ with bounded support and assume that $\frac{n}{n+m} \to r \in (0, 1)$ as $n, m \to \infty$.

Let $A_N \in \R^{s \times d}$ be a PCA transform, fitted on $\xi_1, \ldots, \xi_n, \xi_1', \ldots, \xi_m'$ ($N = m+n$), for some $s \in \{1, \ldots, d\}$.
Let $\Sigma = \Cov(\xi_1)$ with eigenvalues $\lambda_1, \ldots, \lambda_d$ sorted in descending order, and assume that $\lambda_s \neq \lambda_{s+1}$ (if $s < d$).

Then
\begin{align*}
    \sqrt{\frac{nm}{n+m}} \left(\frac{1}{n}\sum_{i=1}^n A_N \xi_i - \frac{1}{m}\sum_{i=1}^m A_N \xi_i' \right) \dto \NN(0, \Sigma')
\end{align*}
as $n, m \to \infty$, where $\Sigma' = \mbox{diag}(\lambda_1, \ldots, \lambda_s)$.
\end{theorem}
Note that the $\lambda_s \neq \lambda_{s+1}$ assumption is only necessary for uniqueness of the limiting distribution -- in practice one can ignore this requirement.

\begin{proof}[Proof of Theorem~\ref{thm-pca}]
    Since $A_N$ is a PCA transformation, $A_N$ is the matrix with the normalized eigenvectors corresponding to the $s$ largest eigenvalues of the empirical covariance matrix $\tilde{\Sigma}_{n,m}$.
    But, due to a weak law of large numbers, $\tilde{\Sigma}_{n,m} \pto \Sigma$ and accordingly $A_N \pto A$ with the population PCA $A$ being the normalized eigenvectors corresponding to the $s$ largest eigenvalues of $\Sigma$ (without loss of generality we can assume the row-wise signs to be determined by some deterministic procedure, and hence for large enough $N$, $A_N$ and $A$  unique e.g. by requiring that the first non-zero entry in the vector be positive).
    
    Due to the same argument as in the proof of Theorem~\ref{thm-pvals} \emph{(i)}, 
    \begin{align*}
        \sqrt{\frac{nm}{n+m}} \left(\frac{1}{n}\sum_{i=1}^n \xi_i - \frac{1}{m}\sum_{i=1}^m \xi_i' \right) \dto \NN(0, \Sigma).
    \end{align*}
    
    Due to a multivariate Slutsky theorem, then
    \begin{align*}
        \sqrt{\frac{nm}{n+m}} &  \left(\frac{1}{n}\sum_{i=1}^n A_N \xi_i - \frac{1}{m}\sum_{i=1}^m A_N \xi_i' \right)
        \\
        & = A_N \sqrt{\frac{nm}{n+m}} \left(\frac{1}{n}\sum_{i=1}^n \xi_i - \frac{1}{m}\sum_{i=1}^m \xi_i' \right)
        \\
        & \dto \NN(0, A\Sigma A^\top).
    \end{align*}
    But as $A$ consists of the orthogonal eigenvectors of $\Sigma$ in descending order of eigenvalues, $A\Sigma A^\top = \Sigma'$.
    
\end{proof}

\section{ADDITIONAL EMPIRICAL ANALYSIS}
\label{sec:additionalexp}

\subsection{Parameters for SCF and ME tests}
For the SCF and ME test, hyperparameters have to be chosen, namely the number of locations/frequencies at which to test, the kernel-selection strategy and whether to optimize over the frequencies/locations or to use a simple heuristic.
We found that if the number of locations/frequencies $J$ is chosen too large, the tests oftentimes strongly violate the significance level.
Hence, we grow $J$ with the number of samples according to what still gives reasonable type-1 error rates.

\paragraph{AM Audio Data}
Here we use the `full` version of the parameter selection from \citep{jitkrittum2016interpretable} for both tests.
Number of frequencies/locations were set to $J = 1$ when $m \in [10, 50]$, $J=3$ for $m \in [75, 150]$ and $J = 10$ for $m \in [200, 1000]$.

\paragraph{Aircraft, Dogs and Birds Data}
For SCF we found the random location initialization without kernel optimization (and hence without data split) to work best.
For ME, due to the high dimensionality, we selected the `grid` version of the parameter optimization; the `full` version did not seem to give considerable improvements above this.
For the Aircraft and Dogs data, we selected $J=1$ frequencies/locations for $m\in[10,50]$ and $J=3$ for $m \in [50, 200]$.
For the Birds data we always use $J=1$ ($m\in[10,60])$.

\paragraph{Facial Expression Data}
Again we use random locations for SCF and grid-search kernel width for ME.
For SCF, we fix $J=1$ for all $m\in[10,200]$.
For ME, we choose $J=1$ for $m\in[10, 50]$, $J=3$ for $m\in[75,100]$ and $J=10$ for $m\in[150,200]$.

\subsection{Image Preprocessing}
\label{sec:preprocessing}
For the deep learning-based methods (DFDA, DMMD \& C2ST), before evaluation, all image data is rescaled to (224, 224) and normalized according to the requirements of the neural network.

For kernel-based tests we found different strategies to work differently well on each data set.
Hence, for the Aircraft, Stanford Dogs and Birds data set, data is rescaled to (48, 48) dimensions and converted to grayscale.
For the facial expression data, images were first cropped to the center (resulting in (462, 462) dimensions) and then rescaled to (96, 96) dimensions; no conversion to grayscale was performed.
We found no increase in power for higher resolution (e.g. (224, 224)).

\subsection{Sensitivity to Significance Level}
\label{sec:sensitivity}
\begin{figure*}[t]
\centering
    \begin{subfigure}[b]{0.45\textwidth}       
        \includegraphics[width=\textwidth]{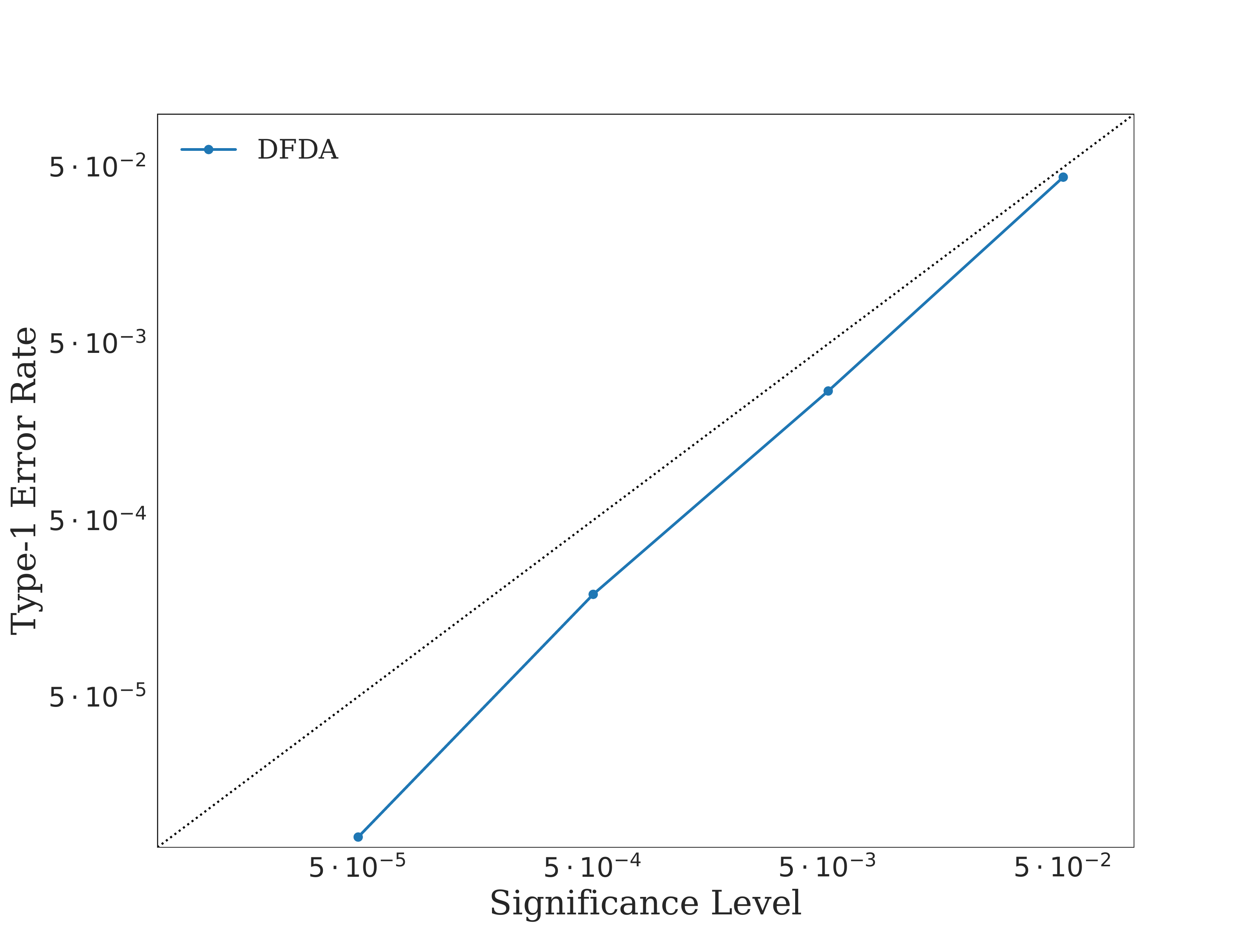}
        \caption{Type-1 error rate at low significance levels, $m=50$.}
        \label{fig:alpha-t1er}
    \end{subfigure}
    \begin{subfigure}[b]{0.45\textwidth}
        \includegraphics[width=\textwidth]{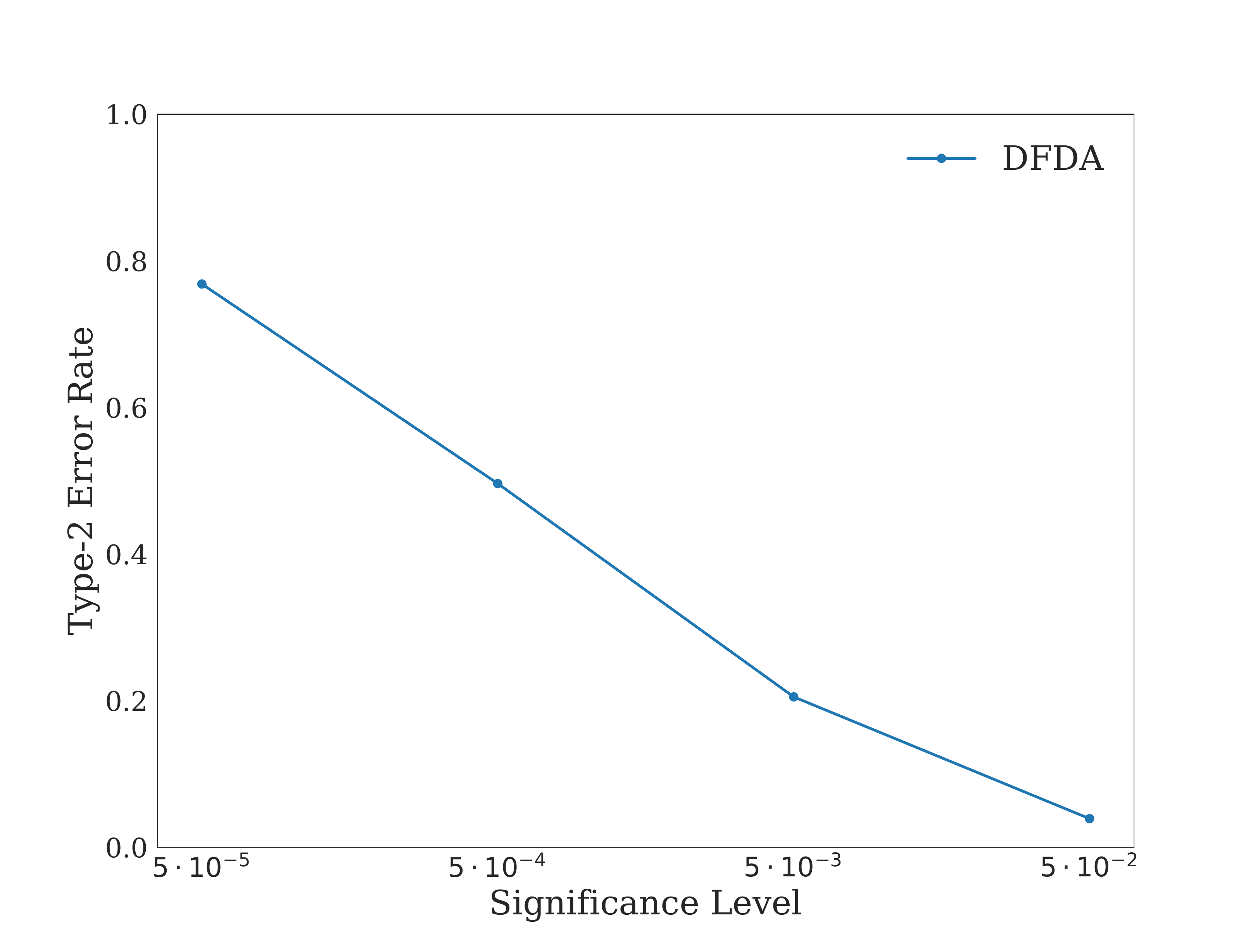}
        \caption{Type-2 error rate at low significance levels, $m=50$.}
        \label{fig:alpha-t2er}
    \end{subfigure}
\caption{Results on the AM audio data for $m=50$ with small significance levels $\alpha$. We show average values over $10^6$ tests, where we fixed the sample size $m$ per population to be equal to 50. (a) Empirical type-1 error rates for small $\alpha$ values consistently lie below the expected type-1 error rate (dotted line). (b) Empirical type-2 error rates. }
\label{fig:alphalevel}
\end{figure*}

    Special care has to be taken if several hypotheses are tested at the same time, leading to a \emph{multiple testing problem}.
    One simple approach to control the so-called familywise error rate (FWER, \citep{lehmann2006testing}), i.e., the probability of at least one wrong rejection of a null hypothesis, is the Bonferroni correction \citep{lehmann2006testing}. 
    The Bonferroni correction divides the original significance level $\alpha$ by the number of tests to be performed.
    Therefore, in many practical settings the significance level for each test will be considerably lower than the ``standard`` values of $0.05$ or $0.01$.
    This represents a problem in practice, since approximating the distribution in the tails usually is more challenging.
    Here we only give results for the asymptotic DFDA distribution, since permutation-based methods do not scale well to very low significance levels.
    Figure~\ref{fig:alpha-t1er} shows that our method controls type-1 error rate at significance levels $5\cdot 10^{-u}$ for $u=2,3,4,5$; Figure~\ref{fig:alpha-t2er} shows that even at small significance levels, the DFDA can still maintain relatively high power.

\subsection{Control of Type-1 Error Rate}
\label{sec:t1er}
Figure~\ref{fig:t1er-appendix} shows that both DMMD and DFDA properly control the type-1 error rate even at low sample sizes.
\begin{figure*}[ht]
\centering
    \begin{subfigure}[b]{0.45\textwidth}       
        \includegraphics[width=\textwidth]{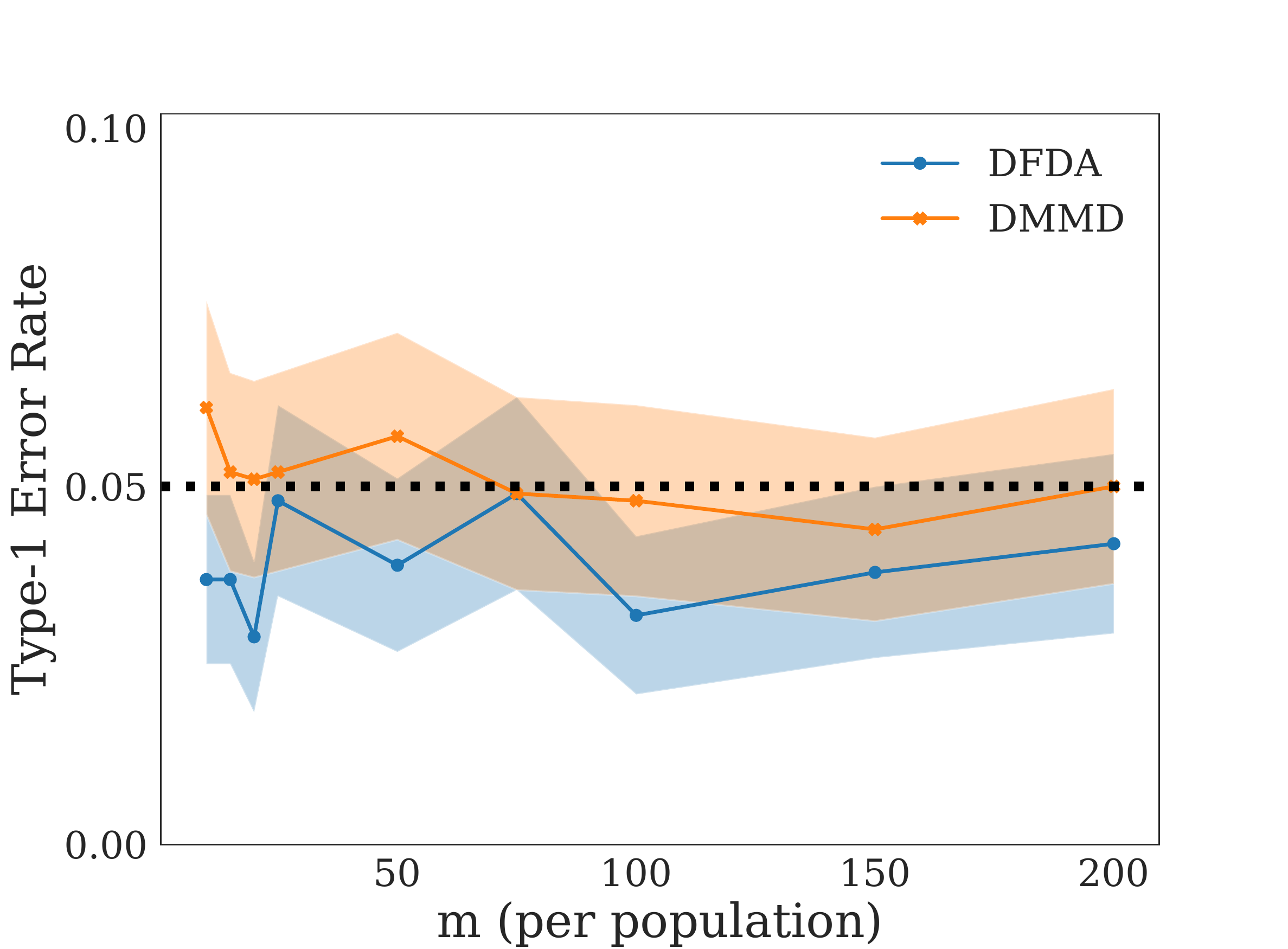}
        \caption{Type-1 error rate on Aircraft data set.}
        \label{fig:planes-t1er}
    \end{subfigure}
    \begin{subfigure}[b]{0.45\textwidth}
        \includegraphics[width=\textwidth]{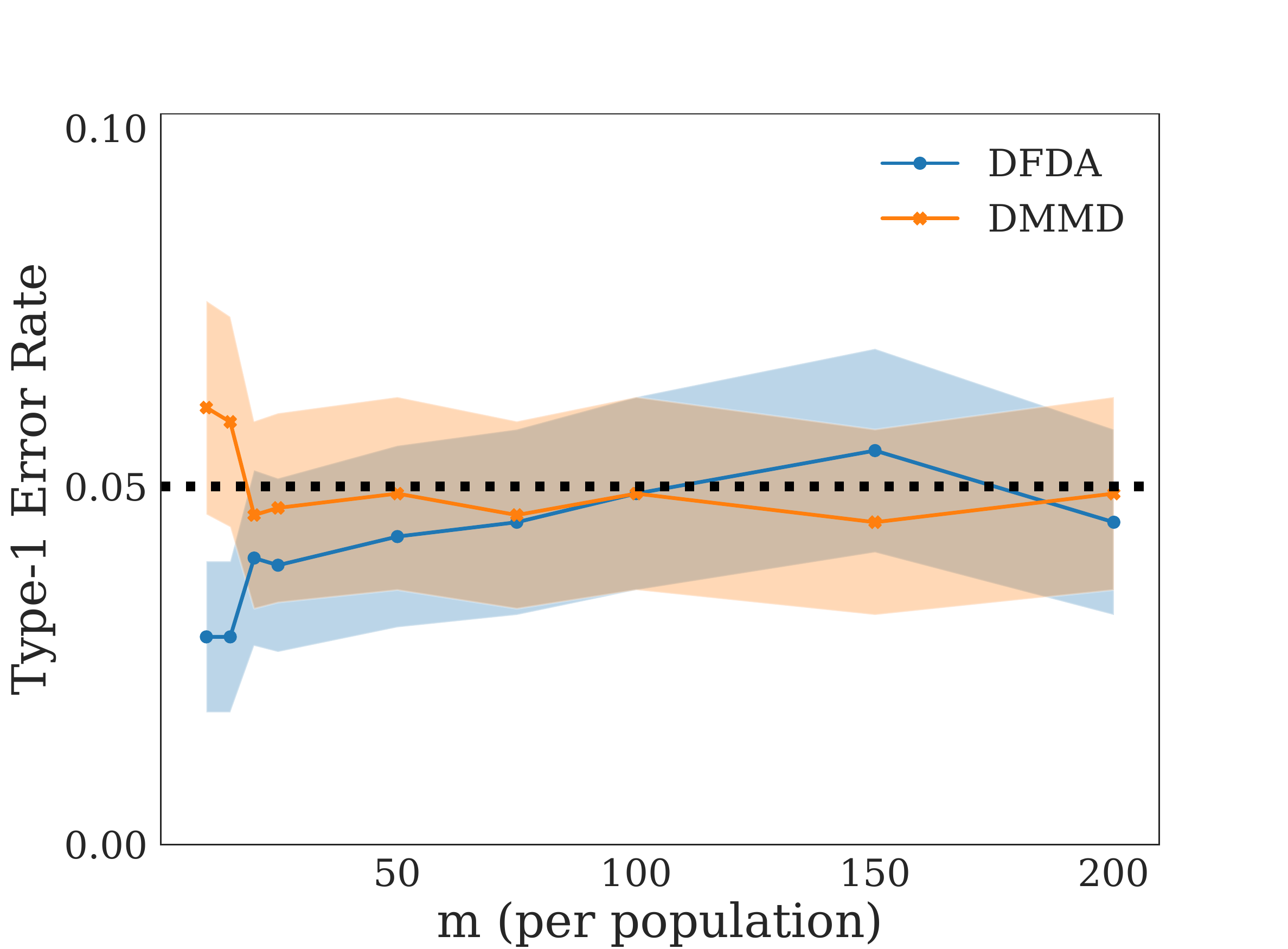}
        \caption{Type-1 error rate on facial expression data set.}
        \label{fig:faces-t1er}
    \end{subfigure}
    \begin{subfigure}[b]{0.45\textwidth}
        \includegraphics[width=\textwidth]{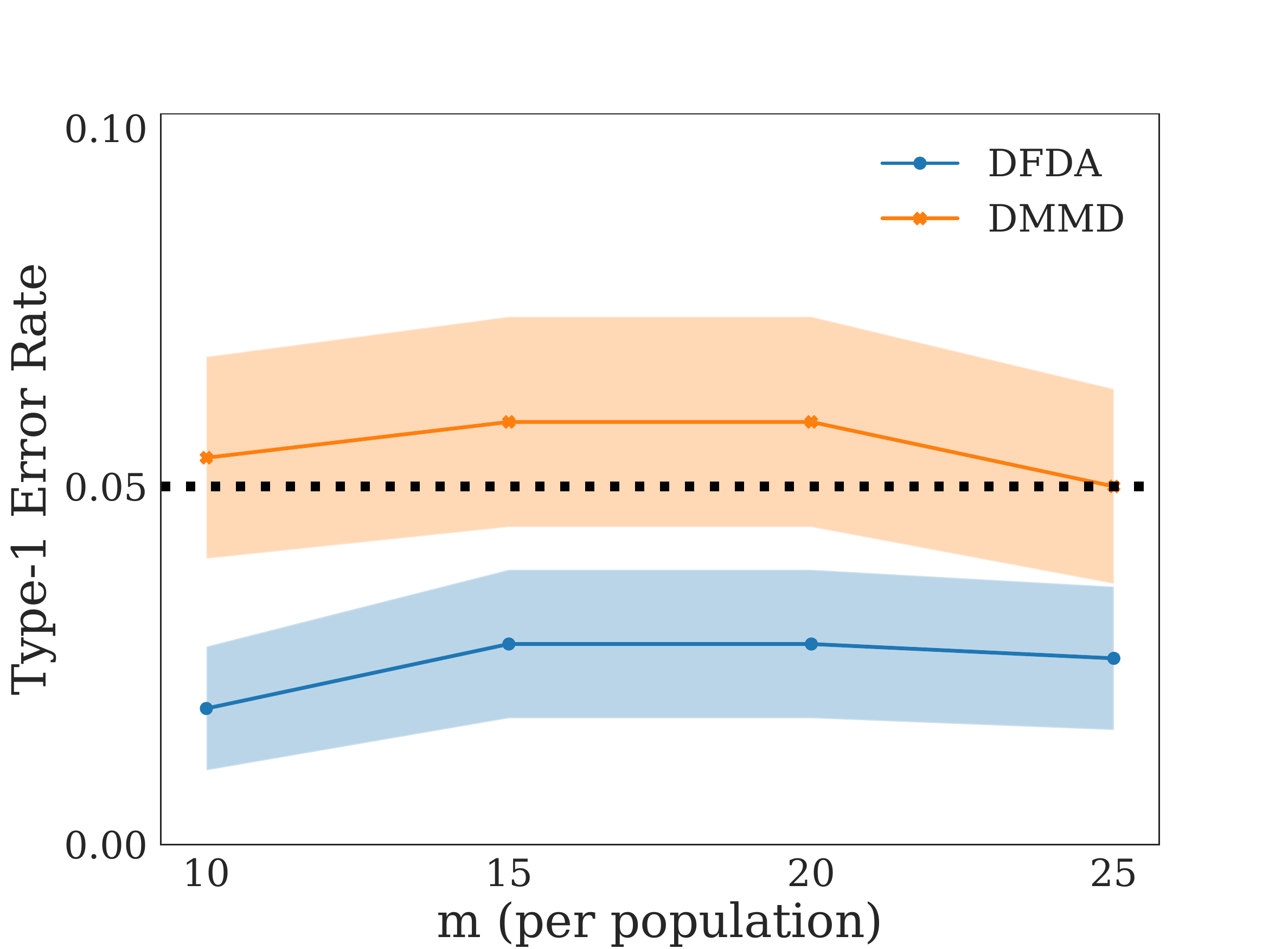}
        \caption{Type-1 error rate on Birds data set.}
        \label{fig:birds-t1er}
    \end{subfigure}
    \begin{subfigure}[b]{0.45\textwidth}
        \includegraphics[width=\textwidth]{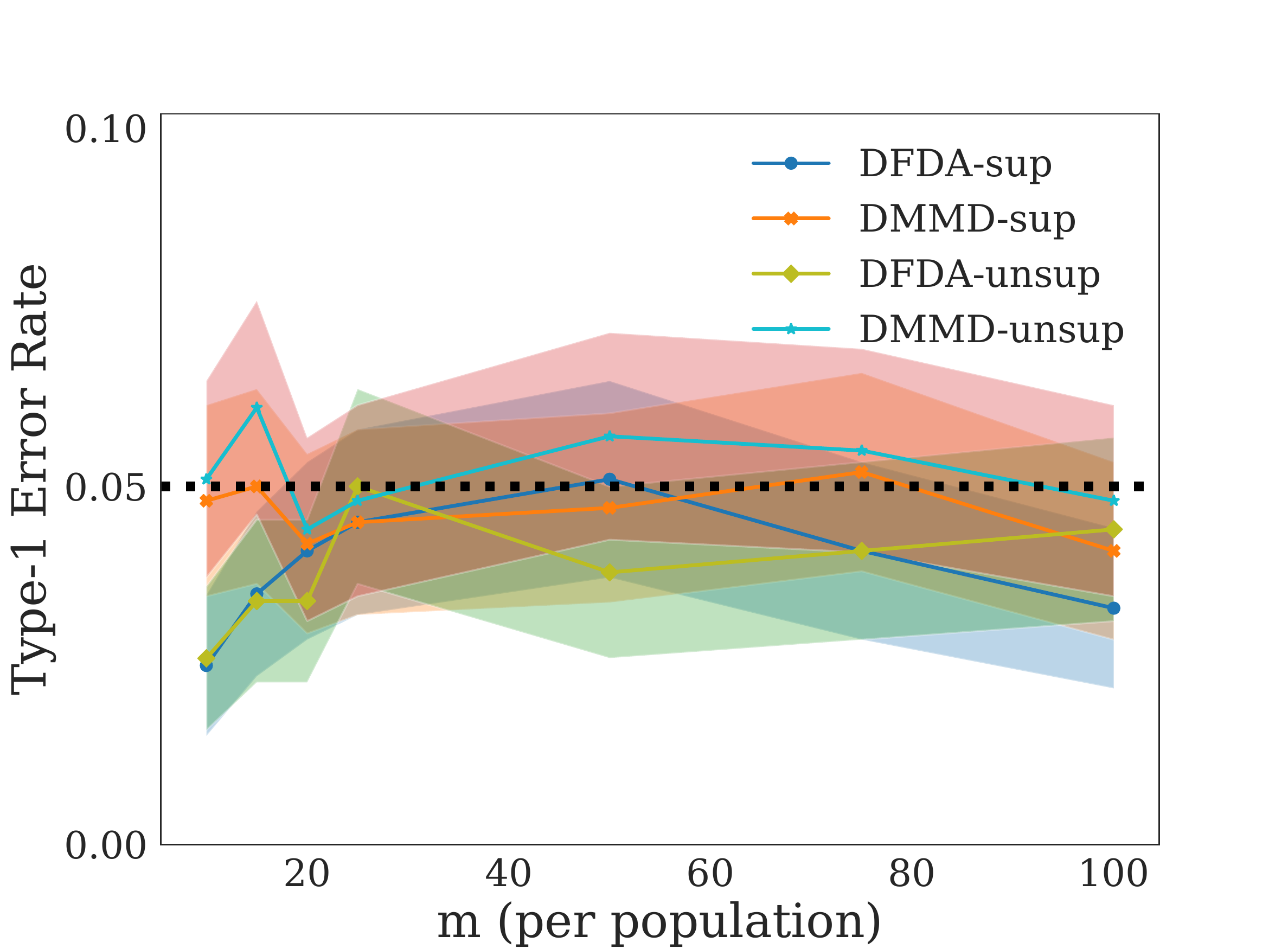}
        \caption{Type-1 error rate on Stanford Dogs data set.}
        \label{fig:dogs-t1er}
    \end{subfigure}
\caption{Empirical control of type-1 error rate on vision data sets.}
\label{fig:t1er-appendix}
\end{figure*}

\subsection{Birds Experiments}
\begin{figure*}[h]
\centering
    \begin{subfigure}[b]{0.48\textwidth}
        \includegraphics[width=\textwidth]{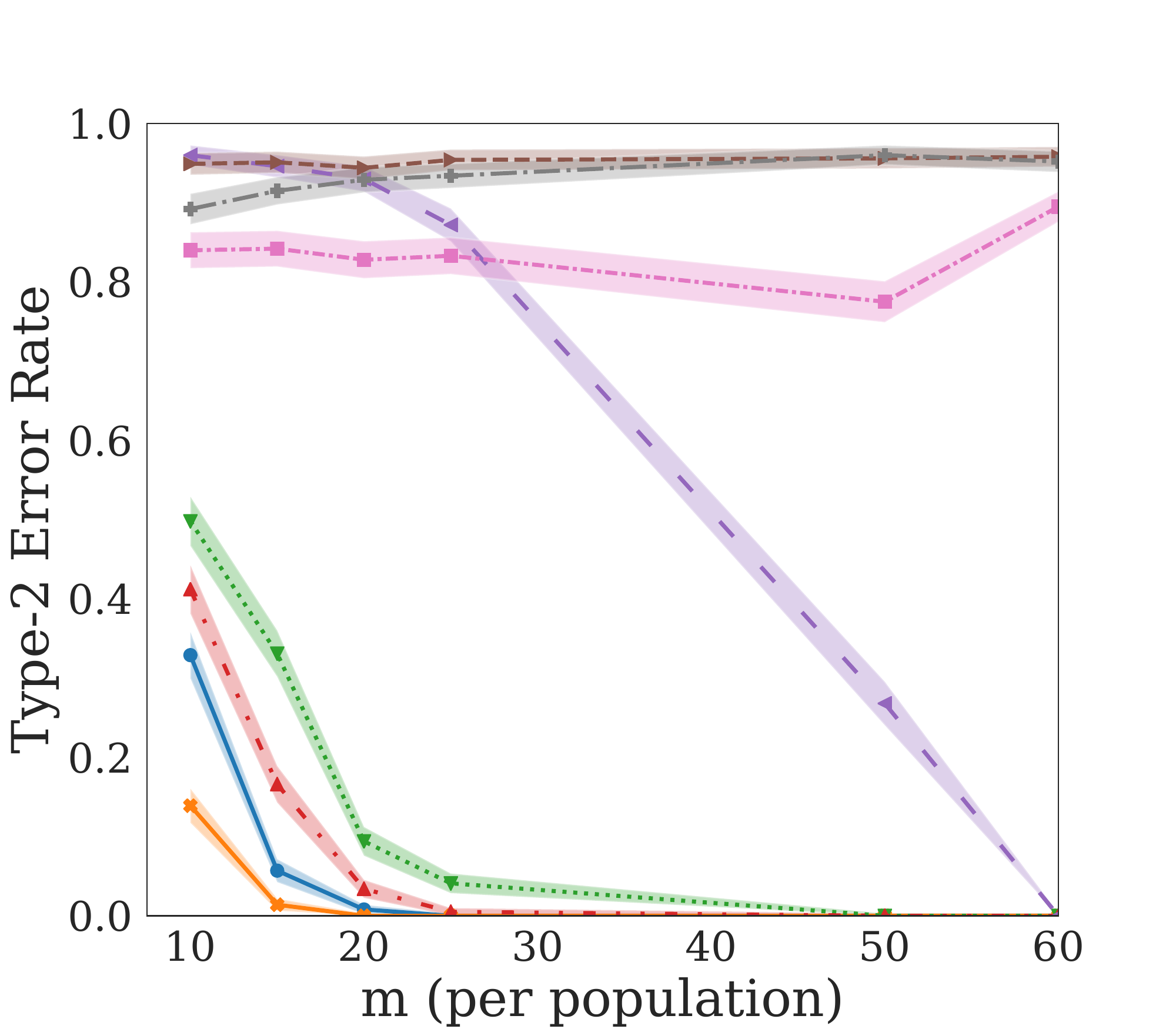}

    \end{subfigure}
    \begin{subfigure}[b]{0.48\textwidth}
        \includegraphics[width=0.6\textwidth]{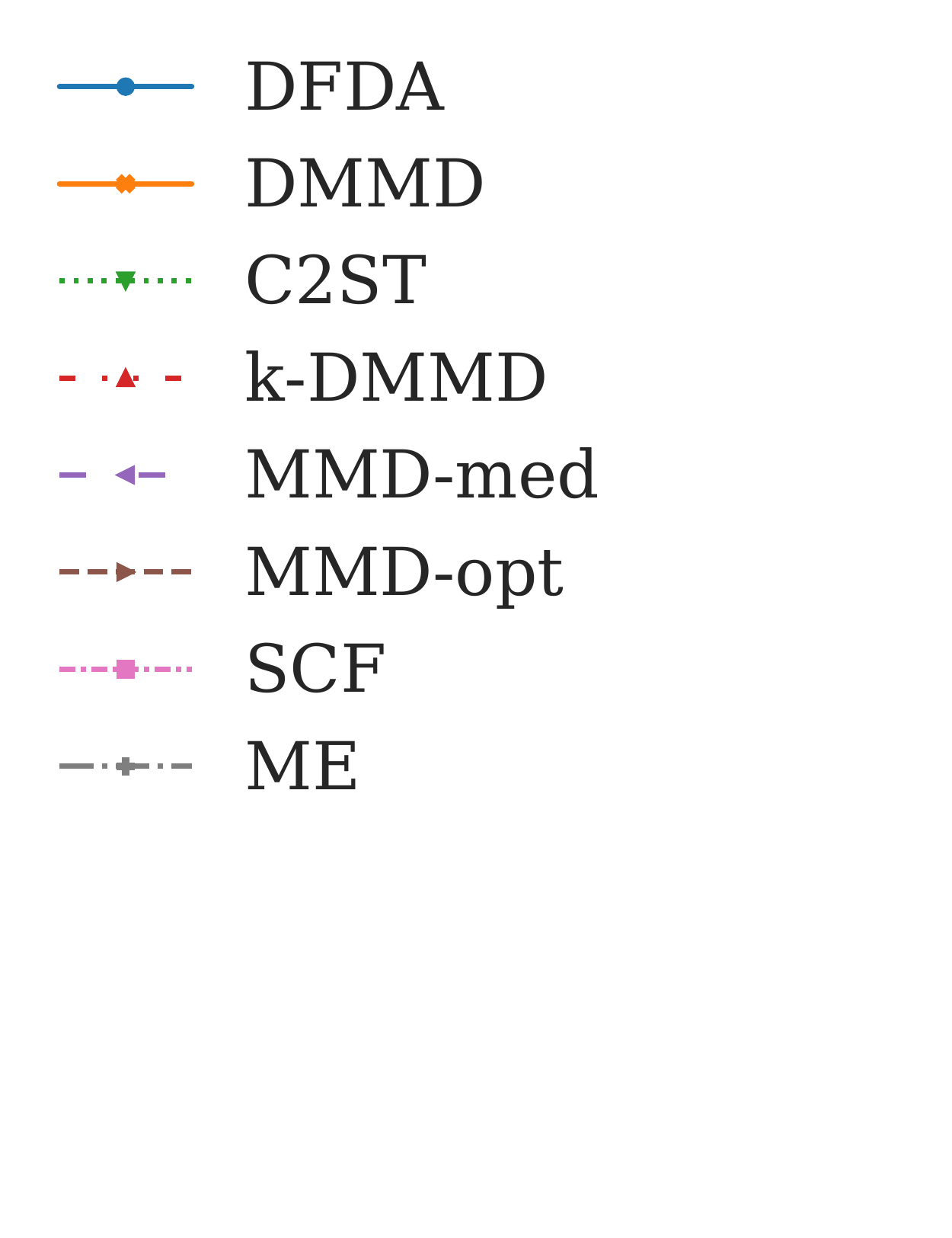}
    \end{subfigure}
            \caption{Type-2 error rate on Birds data set.}
        \label{fig:birds}
\end{figure*}

Here we report results on another fine-grained classification data set, the Caltech-UCSD Birds-200-2011, Caltech-UCSD Birds-200-2011 \citep{WahCUB_200_2011}.
We selected two visually very similar species of birds, namely the ``Blue-winged Warbler`` and the ``Hooded Warbler`` for differentiation.
Results are shown in Figure~\ref{fig:birds}.

\subsection{AM Audio Experiments}
\label{sec:audio-appendix}
Data preprocessing consists of sampling the original audio signal at 8kHz, the resulting AM signal is sampled at 120kHz, and snippets of length 1000 are used for identification.
Gaussian noise with standard deviation 1 is added to the samples after processing.

The model has four one-dimensional convolutional layers, each followed by Batch normalization, a ReLU activation and max-pooling.
The last layer is fully connected, but only used for training the network, i.e., the feature extraction is fully convolutional.
In contrast to the M5 network, we use an input layer with kernel size of 20 instead of 80 and the final global average pooling layer can be removed, to accommodate the significantly smaller input dimension of the audio snippets.
We train the network to classify noisy AM snippets from the remaining songs on the album, with a multi-class cross-entropy loss and a $L_2$-regularization of $10^{-4}$ on all weights; we use the Adam optimizer for this task \cite{kingma2014adam}.

\subsection{Stanford Dogs Experiments}
\label{sec:dogs-appendix}

Table~\ref{tab:cae-dogs} shows the convolutional autoencoder architecture used in the experiments on the Stanford Dogs data set.
The autoencoder was trained to optimize multi-scale structural similarity between input and output images.

The supervised training was performed with a network with the same encoder as in Table~\ref{tab:cae-dogs} and a fully connected layer on top, to classify the remaining 118 dog breeds.
Again, we use the multi-class cross-entropy loss.

For both the supervised and the unsupervised task we use the Adam optimizer and $L_2$ regularization of size $10^{-4}$.

\subsection{KDEF Experiments}
Note that \citet{jitkrittum2016interpretable} and \citet{lopez2016revisiting} only compared tests that use train/test splits.
Hence, results therein are reported for $n_{te}$, which is the size of the test set of each sample, i.e. $n_{te} = \frac{1}{2} m$ in our case ($n_{te} = 201$ corresponds to $m = 402$).

\subsection{Imagenet Training}
For the aircraft, facial expression, and birds data set we use a ResNet-152, trained on the whole ILSVRC  2012 data set.
Instead of training this network ourselves, we use the parameters and implementation provided in the PyTorch deep learning library \cite{paszke2017automatic}.

\begin{table}[t]
  \caption{Architecture of the convolutional autoencoder used for the Stanford Dogs experiments. For Conv and ConvTranspose layers, $[3\times3, f]$ denotes $f$ $3\times 3$ filters. Activation functions are always ReLUs except for the last convolutional layer (tanh) and the last ConvTranspose layer (sigmoid). After each Conv and ConvTranspose operation, a BatchNorm \citep{ioffe2015batch} operation was used. The output of the encoder part was used as feature map in our tests.}
  \label{tab:cae-dogs}
  \centering
  \begin{tabular}{lc}
    \toprule
\multicolumn{2}{c}{Input: $(3, 224, 224)$ image} \\
\cmidrule{1-2}
\multicolumn{2}{c}{Encoder} \\
\cmidrule{1-2}
Conv & $[3\times3, 40]$ \\
MaxPool & $[2\times 2]$ \\
Conv & $[3\times3, 80]$ \\
MaxPool  & $[2\times 2]$\\
Conv & $[3\times3, 160]$ \\
MaxPool & $[2\times 2]$\\
Conv & $[3\times3, 240]$ \\
MaxPool &  $[2\times 2]$\\
Conv & $[3\times3, 360]$ \\
MaxPool &  $[2\times 2]$\\
Conv &  $[3\times3, 2048]$ \\
MaxPool &   $[2\times 2]$\\
\cmidrule{1-2}
\multicolumn{2}{c}{Decoder}\\
\cmidrule{1-2}
ConvTranspose & $[3\times3, 360]$ \\
Upsample &  $[2\times 2]$ \\
ConvTranspose & $[3\times3, 240]$ \\
Upsample &  $[2\times 2]$ \\
ConvTranspose & $[3\times3, 160]$ \\
Upsample &  $[2\times 2]$ \\
ConvTranspose & $[3\times3, 80]$ \\
Upsample &  $[2\times 2]$ \\
ConvTranspose & $[3\times3, 40]$ \\
Upsample & $[2\times 2]$ \\
ConvTranspose & $[3\times3, 3]$\\
    \bottomrule
  \end{tabular}
\end{table}
\subsection{MRI Scan Preprocessing and Experiments}
\label{sec:mri-appendix}

 The T1 MRI scans acquired through the MP-RAGE protocol were selected from GSP and ADNI. The scans were standardized to $(256, 256, 256)$ and cropped to $(96, 96, 96)$ dimensions with isotropic voxels of 1mm. Model architecture is shown in table~\ref{table:3DCAE}. The model was trained for 400 epochs on 1413 MRI scans from GSP. The loss function was set to the mean squared error and the batch size was set to one. No MRI scans from ADNI was used for training.
 
 In our experiments, the \emph{APOE} gene was used since it is known to be a risk factor for Alzheimer's disease; in practice, when one does not know which locus to test, a multistep-approach such as the one developed by \citet{mieth2016combining} can be used to create a selection of candidate loci.

\begin{table}[t]
  \caption{Architecture of the 3D convolutional autoencoder for the MRI data. For Conv and ConvTranspose layers, $[3\times3\times3, s, f]$ denotes $f$ $3\times3\times3$ filters with strides of $s$. Activation functions are always ReLUs except for the last convolutional layer (linear). All convolutional operations are done without padding. The output of the encoder ($1024$ dimensions) is used as feature map in our tests.}
  \label{table:3DCAE}
  \centering
  \begin{tabular}{ll}
    \toprule
    \multicolumn{2}{c}{Input: $(96, 96, 96)$ MRI scan}  \\
    \midrule
\multicolumn{2}{c}{Encoder} \\
\cmidrule(r){1-2}
Conv & $[3\times3\times3, 1, 8]$ \\
Conv & $[2\times2\times2, 2, 16]$\\
Conv & $[3\times3\times3, 1, 32]$\\
Conv & $[2\times2\times2, 2, 64]$\\
Conv & $[2\times2\times2, 2, 128]$\\
Conv & $[2\times2\times2, 2, 256]$\\
Conv & $[2\times2\times2, 2, 256]$\\
Dense\\
\cmidrule(r){1-2}
\multicolumn{2}{c}{Decoder}\\
\cmidrule(r){1-2}
Dense \\
Conv & $[3\times3\times3, 1, 256]$\\
ConvTranspose & $[2\times2\times2, 2, 256]$\\
ConvTranspose & $[2\times2\times2, 2, 128]$\\
ConvTranspose & $[2\times2\times2, 2, 64]$\\
ConvTranspose & $[2\times2\times2, 2, 32]$\\
Conv & $[3\times3\times3, 1, 16]$\\
ConvTranspose & $[2\times2\times2, 2, 8]$\\
Conv & $[3\times3\times3, 1, 1]$\\
    \bottomrule
  \end{tabular}
\end{table}

\section{CODE AND DATA}
We provide an implementation of our methods at \url{https://github.com/mkirchler/deep-2-sample-test}.

All 2D imaging and audio data can be downloaded from the following sources:
\begin{itemize}
    \item Audio data: \url{http://dl.lowtempmusic.com/Gramatik-TAOR.zip}
    \item Aircraft data: \url{http://www.robots.ox.ac.uk/~vgg/data/fgvc-aircraft/archives/fgvc-aircraft-2013b.tar.gz}
    \item Facial Expression data: \url{http://kdef.se/index.html}
    \item Stanford Dogs data: \url{http://vision.stanford.edu/aditya86/ImageNetDogs/images.tar}
    \item Birds data: \url{http://www.vision.caltech.edu/visipedia-data/CUB-200-2011/CUB_200_2011.tgz}
\end{itemize}
For MRI imaging data access to data has to be granted by the releasing institutions, see
\begin{itemize}
    \item GSP: \url{https://www.neuroinfo.org/gsp}
    \item ADNI: \url{http://adni.loni.usc.edu/data-samples/access-data/}
\end{itemize}

\end{document}